\documentclass[11pt]{article}
\usepackage[margin=1in]{geometry}
\usepackage{setspace}
\usepackage{graphicx}
\usepackage{natbib}

\usepackage{amsmath, amsthm, amssymb}
\usepackage{authblk}
\usepackage{mathtools}
\usepackage{times}
\usepackage{wrapfig}

\usepackage{fancyhdr}
\usepackage{diagbox}
\usepackage{xcolor} % changed from color to xcolor (2021/11/24)
\usepackage{mwe}
\usepackage{algorithm}
\usepackage{algorithmic}
\usepackage{eso-pic} % used by \AddToShipoutPicture
\usepackage{forloop}
\usepackage{url}
\usepackage{multicol}
\usepackage{microtype}
\usepackage{array}
\usepackage{subcaption}
\usepackage{booktabs} % for professional tables
\usepackage{makecell}
\usepackage{varwidth}
\usepackage{amsmath}
\usepackage{cprotect}
\usepackage{fancyvrb}
\usepackage{indentfirst}
\usepackage{multirow}
\usepackage{enumitem}%http://ctan.org/pkg/enumitem
\usepackage{amsfonts}

\newcommand{\G}{\mathcal{G}}
\newcommand{\V}{\mathcal{V}}
\newcommand{\E}{\mathcal{E}}

\newcommand{\R}{\mathbb{R}}
\newcommand{\B}{\mathcal{B}}

\newcommand{\hatThetat}[1]{\hat{\theta}_L^{(#1)}}
\newcommand{\normsquare}[1]{\|#1\|_2^2}
\newcommand{\commentalign}[1]{[\text{#1}]}
\newcommand{\EE}{\mathbb{E}}
\newcommand{\PP}{\mathbb{P}}

\newcommand{\VI}{\text{VI}[F,\bTheta]}

\newcommand{\EEXY}[1]{\EE_{X,Y}\{#1\}}
\newcommand{\Femp}{\widehat F}

\newcommand{\bTheta}{\Theta}

\newcommand{\YTheta}{Y({\theta_L})}
\newcommand{\SVI}{\Verb|SVI|}

\newcommand{\condexp}[2]{\EE[#1|#2]}

\newcommand{\etaT}[1]{\eta^{\intercal}(#1)}

\newcommand{\Xstar}[1]{X^*_{#1}}
\newcommand{\thetastar}[1]{\theta^*_{#1}}

\newcommand{\thetalower}[1]{\theta_{#1}}
\newcommand{\Lc}{\mathcal{L}}
\newlength{\tempdima}
\newcommand{\rowname}[1]% #1 = text
{\rotatebox{90}{\makebox[\tempdima][c]{\textbf{#1}}}}

\renewcommand{\thesubfigure}{\alph{subfigure}}
\newcommand{\mycaption}[1]% #1 = caption
{\refstepcounter{subfigure}\textbf{(\thesubfigure) }{\ignorespaces #1}}
\newcolumntype{P}[1]{>{\centering\arraybackslash}p{#1}}

\usepackage{xcolor}

\theoremstyle{plain}
\newtheorem{theorem}{Theorem}[section]
\newtheorem{proposition}[theorem]{Proposition}
\newtheorem{lemma}[theorem]{Lemma}

\theoremstyle{definition}

\theoremstyle{remark}
\newtheorem{remark}[theorem]{Remark}

\begin{document}

\title{An alternative approach to train neural networks
using \\ monotone variational inequality}
\author[1]{Chen Xu\thanks{cxu310@gatech.edu}}
\author[2]{Xiuyuan Cheng\thanks{xiuyuan.cheng@duke.edu}}
\author[1]{Yao Xie\thanks{yao.xie@isye.gatech.edu}}
\affil[1]{{\small H. Milton Stewart School of Industrial and Systems Engineering, Georgia Institute of Technology.}}
\affil[2]{{\small Department of Mathematics, Duke University}}

\maketitle

\begin{abstract}
We propose an alternative approach to neural network training using the monotone vector field, an idea inspired by the seminal work of Juditsky and Nemirovski \citep{VI_est} developed originally to solve parameter estimation problems for generalized linear models (GLM) by reducing the original non-convex problem to a convex problem of solving a monotone variational inequality (VI). Our approach leads to computationally efficient procedures that converge fast and offer guarantee in some special cases, such as training a single-layer neural network or fine-tuning the last layer of the pre-trained model. Our approach can be used for more efficient fine-tuning of a pre-trained model while freezing the bottom layers, an essential step for deploying many machine learning models such as large language models (LLM). We demonstrate its applicability in training fully-connected (FC) neural networks, graph neural networks (GNN), and convolutional neural networks (CNN) and show the competitive or better performance of our approach compared to stochastic gradient descent methods on both synthetic and real network data prediction tasks regarding various performance metrics.
\end{abstract}
\section{Introduction}

Neural Network (NN) training \citep{Duchi2010AdaptiveSM,pmlr-v28-sutskever13,kingma2014adam,Ioffe2015BatchNA} is the essential process in the study of deep models. Optimization guarantee for training loss, as well as generalization error, have been obtained with over-parameterized networks \citep{neyshabur2014search,mei2018mean,arora2019exact,arora2019fine,allen2019convergence,du2019gradient}. However, due to the inherent non-convexity of loss objectives, theoretical developments are still diffused and lag behind the vast empirical successes.

Recently, the seminal work \citep{VI_est} presented a somehow surprising result that some non-convex issues can be circumvented in special cases by problem reformulation. In particular, it was shown that when estimating the parameters of the GLM, instead of minimizing a least-square loss function, which leads to a non-convex optimization problem, no guarantees can be obtained for global convergence nor model recovery, one can reformulate the problem as solving a monotone VI, a general form of convex optimization. The reformulation through monotone VI leads to performance guarantees and computationally efficient procedures.

In this paper, inspired by \citep{VI_est} and the fact that certain GLM (such as logistic regression) can be viewed as a layer in neural networks, we consider a new scheme for neural network training based on monotone VI. Our approach is a drastic departure from the widely used gradient descent algorithm for neural network training --- we replace the gradient of a loss function with a constructed monotone operator to achieve faster convergence, which we demonstrate empirically, and further guaranteed convergence in some special cases, such as one layer neural network or fine-tuning the last layer of a neural network.

Our approach can lead to a more efficient fine-tuning of a pre-trained neural network model: training the last layer and ``freezing'' the rest of the layers, with fast convergence and guarantee. Fine-tuning has been shown to be effective in leveraging pre-trained models on large datasets for a similar or related task \citep{liu2022few}. Fine-tuning is a common and essential practice in large language models \citep{ding2023parameter} to improve performance over the unmodified pre-trained model on downstream tasks. For some architectures, such as CNN, it is common to keep the earlier layers (those closest to the input layer) frozen because they capture lower-level features. In contrast, upper layers often focus on high-level features that can be more related to the downstream task. In our paper, we demonstrate (i) for fine-tuning the last layer, we can establish training and prediction guarantees -- see Section \ref{sec:guarantee}; (ii) for fine-tuning the top few (more than one) layers, through extensive numerical studies on synthetic and real-data in Section \ref{sec:expr_main}, we demonstrate the faster convergence to a local solution by our approach relative to gradient descent in a comparable setup. 

To our knowledge, the current paper is the first to study monotone VI for training neural networks. The proposed \SVI, as a general way of modifying the parameter update scheme in NN training, can be applied to various deep architectures. In this work, beyond FC neural networks, we especially study monotone VI training for node classification in GNN \citep{Wu2019ACS,Pilanci2020NeuralNA}, due to the ubiquity of network data and the importance of network prediction problems. We also study monotone VI training for image classification.

In summary, our technical contributions include: 
\begin{itemize}
    \item Develop a general and practical algorithm for training neural networks using vector field constructed by monotone VI, called \textit{stochastic variational inequality} (\SVI). Our work demonstrates the potential power of training neural networks by monotone VI, the idea initially introduced in \cite{VI_est} to estimating statistical GLM. The algorithm provides a fundamentally different but easy-to-implement first-order alternative from the commonly used stochastic gradient descent (SGD) of the empirical loss function. 
    
    \item The computation cost per step is similar between \SVI{} and gradient-based methods. However, training by monotone VI can lead to guarantee in the special case of last-layer training and faster empirical convergence, which is particularly valuable for fine-tuning pre-trained neural network models, as demonstrated by our numerical experiments. 
    
    \item Compare \SVI \ with widely-used SGD methods to demonstrate that \SVI \ is flexible on various tasks and competitive against SGD, especially the improved efficiency in the early stage of training.
\end{itemize}

{\it Literature.} monotone VI
has been studied in priors mainly in the optimization problem context \citep{kinderlehrer2000introduction,Facchinei2003FiniteDimensionalVI}, which can be viewed as the most general problems with convex structure \citep{juditsky2022well}. More recently, VI has been used to solve min-max problems in Generative Adversarial Networks \citep{liu2021first} and reinforcement learning \citep{kotsalis2022simple}. In particular, our theory and techniques are inspired by \citep{VI_est}, which uses strongly monotone VI for signal recovery in GLM. In contrast to their work, we thoroughly investigate using VI to train multi-layer NN and address cases when VI may not be strongly monotone. 

Meanwhile, our techniques bear similarity to works on ``matching loss'' \citep{amid2022locoprop} but are fundamentally different as we leverage monotone VI theory and lead to performance guarantees. 
On the other hand, we emphasize the difference from \citep{Pilanci2020NeuralNA}, which views two-layer NNs as convex regularizers: we focus on model recovery rather than change variables to convexify the loss minimization. In addition, our \SVI \ extends beyond two-layer networks.
Furthermore, a recent work \citep{Combettes2023VINN} proposed a similar VI-inspired approach in training neural networks, with corresponding convergence guarantees of model parameters. In contrast to our \SVI{}, their proposed method cannot train multiple layers simultaneously, and the technique only demonstrated improved empirical performance over gradient-based methods on training the last layer of a deep neural network.

%\vspace{0.1in}
The rest of the work proceeds as follows. Section \ref{setup} introduces general NN setup and provides preliminaries of monotone VI. Section \ref{sec:VI_training} provides the \SVI \ algorithm for training one/last-layer neural networks. Section \ref{sec:guarantee} presents various theoretical guarantees including the prediction performance. Section \ref{sec:heuristic} proposes a heuristic algorithm for training multiple layers, which we demonstrate to work well using numerical examples. Section \ref{sec:expr_main} compares \SVI \ and the SGD on various synthetic and real-data examples to illustrate the potential benefits of \SVI. Section \ref{sec:conclusion} concludes the paper. Appendix \ref{sec:theory_append} contains all proofs.

\section{Problem setup}\label{setup}

Section \ref{sec:NN_notation} introduces the general neural network model and GNN notations, and Section \ref{sec:VI_prelim} introduces the monotone VI preliminaries.

\subsection{General neural networks} \label{sec:NN_notation}

Given $N$ training samples $\{(X^{(i)},Y^{(i)}\}_{i=1}^N$, where the $d$-dimensional features $X^{(i)} \in \mathbb R^d$, and $Y^{(i)}$ are scalars, which can be real-valued or categorical valued, general neural network learning aims to fit a function in the data. Suppose the neural network has $L$-layers, then
\[X_{l+1}=\phi_{l}(g_{l}(X_{l},\theta_{l})), \quad l = 1, \ldots, L-1,\] denote the nonlinear feature transformation from the previous layer, and $X_1=X \in \mathbb R^d$ is the input with $d$-dimensional features; $\theta=\{\theta_1,\ldots,\theta_L\}$ denotes model parameters, each $g_l$ is the linear ``pre-activation.'' For example, for fully connected layer,  $g_l(x, \theta_l) = W_l^T x + b_l$, and $\theta_l = \{W_l, b_l\}$. Different linear pre-activations can be used depending on the data type, including CNN and GNN layers. Each $\phi_l$ denotes the activation function at layer $l$, for example, the sigmoid function $\phi_l(z) =1/(1+e^{-z})$. 
The final output of the neural network is 
\[Y = \phi_L (g_L(X_{L},\theta_L)).\]
The neural network is a function $f = \phi_L \circ g_L \circ \cdots \circ \phi_1 \circ g_1$ with parameter $\theta$, as illustrated in Fig. \ref{nn}, and  $\circ$ in \eqref{equiv} denotes function composition.

For real-valued response, the neural network is typically learned by minimizing the mean-square error loss:
\begin{equation}
    \mathcal L(\theta):= \frac 1N\sum_{i=1}^N (Y^{(i)} - f(X^{(i)}, \theta))^2.
    \label{LS}
\end{equation}
Due to the highly non-linear $f$, the problem is non-convex.
Training neural networks is commonly done by gradient descent, SGD, or other first-order methods using the gradient of the loss function $\mathcal L(\theta)$ in \eqref{LS} with respect to $\theta$.

\begin{figure}[!h]
\centering
\includegraphics[width=\linewidth]{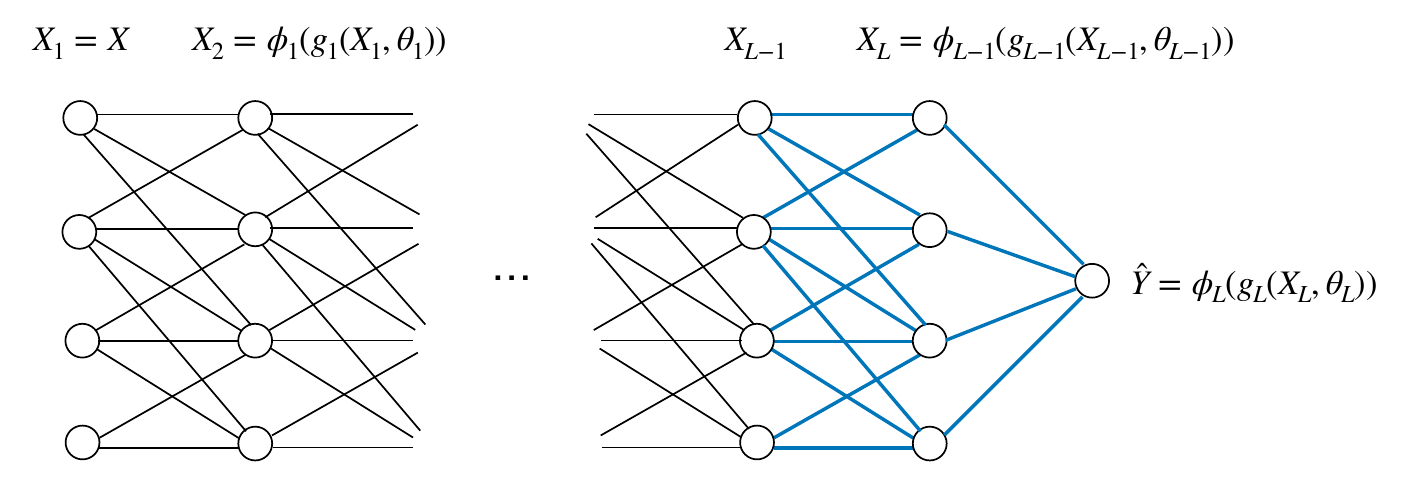}
\cprotect \caption{Diagram to show fine-tuning by training the last two layers for a neural network. Black lines indicate frozen parameters and dark green lines indicate parameters to be fine-tuned by training; in this plot, the last two layers are to be trained.
}
\label{nn}
\end{figure}

\subsection{Preliminaries of monotone VI}\label{sec:VI_prelim}
We now introduce monotone VI. Given a parameter set $\bTheta \subset \R^{p}$, we call a continuous mapping (operator) $F: \bTheta \rightarrow \R^p$ \textit{monotone} if for all $\theta_1, \theta_2 \in \bTheta$, 
$
\langle F(\theta_1)-F(\theta_2),\theta_1-\theta_2 \rangle \geq 0
$ \citep{VI_est}.
The operator is called \textit{strongly monotone} with modulus $\kappa > 0$ if for all $\theta_1, \theta_2 \in \bTheta$, 
\begin{equation}\label{eq:modulus}
    \langle F(\theta_1)-F(\theta_2),\theta_1-\theta_2 \rangle \geq \kappa \|\theta_1-\theta_2\|^2_2.
\end{equation}
If $F \in C^1(\bTheta)$ (i.e., continuously differentiable on $\bTheta$) and $\bTheta$ is closed and convex with non-empty interior, (\ref{eq:modulus}) holds if and only if $\lambda_{\min}(\nabla F(\theta))\geq \kappa$ for all $\theta \in \bTheta$, where $\lambda_{\min}(\cdot)$ denotes the minimum eigenvalue of a square matrix.

Now, for a monotone operator $F$ on $\bTheta$, the problem called $\VI$ is to solve the following
\[
\mbox{Find } \bar{\theta} \in \bTheta, ~~s.t. ~~ 
\langle F(\bar{\theta}), \theta-\bar{\theta} \rangle \geq 0, \forall \theta \in \bTheta. \qquad \VI
\]
It is known that if $\bTheta$ is compact, then $\VI$ has at least one solution \citep[Theorem 4.1]{outrata2013nonsmooth}. In addition, when $\bTheta$ is a convex set, then $\VI$ is a convex problem, and  if $\kappa>0$ in (\ref{eq:modulus}), then $\VI$ has exactly one solution \citep[Theorem 4.4]{outrata2013nonsmooth}. Under mild regularity assumptions, the solution can be solved efficiently to high accuracy, using various iterative schemes such as SGD or the accelerated versions by replacing the gradient with the vector field defined by monotone operator \citep{VI_est}. 

\section{Monotone VI for last-layer training}\label{sec:VI_training}

We now present the monotone VI-based algorithms for training neural networks. We start with simple one/last-layer network training with monotone VI, where the framework and techniques directly come from \citep{VI_est}, highlighting the similarity/difference with SGD. We then generalize to training multiple-layer neural networks in Section \ref{sec:heuristic}.

The motivation for using monotone VI for training is by observing that in each NN layer, the weighted summation and then passing through a monotone non-linear activation can be related to a GLM. Suppose the conditional expectation $\EE[Y|X]$ of the response vector is modeled by an $L$-layer neural network $f(X,\theta)$. We make the following assumption
\begin{center}
There exists an ``ideal model''  $\theta^*=\{\thetastar{1},\ldots,\thetastar{L}\}$ such that $\EE[Y|X]=f(X,\theta^*)$.
\end{center}
where the expectation is with respect to the conditional distribution of $Y$ on $X$; i.e., the best prediction given $X$ on $Y$ that minimizes the mean square error is specified by $f$ with the ideal parameter. Such an assumption can be assumed to be reasonable if $f$ being parameterized by expressive enough neural networks. Note that this can hold for both continuous and categorical responses.

Let us first consider a single-layer neural network (i.e., $L=1$) 
\[g(X,\theta)=\eta(X)\theta, \quad \mathbb E[Y|X] = \phi(\eta(X)\theta).\] for a given feature transformation $\eta$ from the input $X$. We construct the monotone operator $F$ as, inspired by solving GLM in \citep{VI_est}:
\begin{equation}\label{eq:operatr}
    F(\theta):=\EEXY{\etaT{X}[\phi (\eta(X)\theta)-Y]},
\end{equation}
where $z^\intercal$ denotes the transpose of $z$. It can be shown that when $\phi$ is a monotone function, $F$ is monotone \cite{VI_est}.
We will further explain a few key properties of $F$ in Sec. \ref{sec:guarantee}, Lemma \ref{lem1}. 

Given training samples, we can form a sample version of $F$ using training data, denoted by $\widehat F$. Then, we can train the model using a vector field by monotone operator, which we call the stochastic variational inequality (\SVI)  algorithm
\begin{equation}\label{eq:VI_training}
\theta \leftarrow \theta - \gamma \widehat F(\theta),%
\end{equation}
where $\gamma>0$ is the step-size.  
Due to the monotone vector field property of $F$, we can prove convergence and guarantee of the training iteration \eqref{eq:VI_training} as we show in Section \ref{sec:guarantee}. 
\SVI~ based on \eqref{eq:VI_training} differs from SGD, which uses the gradient of a specific loss objective. In \SVI, the iteration follows a vector field constructed using the monotone operator $\widehat F$, which does not need to correspond to the gradient of a loss function. Nevertheless, it is known when we minimize a convex objective, $\widehat F$ corresponds to the gradient of the objective function: in Section \ref{sec:equivalence}, we show when minimizing the cross-entropy loss if $\phi$ is either the sigmoid or softmax, $\widehat F$ corresponds to the gradient with respect to parameter $\theta$ whereby \eqref{eq:VI_training} coincides with SGD.

In the context of fine-tuning the last layer, given an input $X$  to the neural network, $\eta(X)=X_L$, i.e., the feature extracted by previous $L-1$ ``frozen'' layers and 
\begin{equation}\label{eq:hidden_feature}
X_L =  \phi_{L-1}\circ g_{L-1} \circ \cdots \phi_{1}\circ g_{1}  (X, \theta).
\end{equation}
This way, we can cast the one/last-layer training as solving a monotone VI, and provide a prediction bound for $\mathbb{E}_X \{\|f(X,\widehat \theta)-f(X, \theta^*)\|_2^2\}$ for estimator $\widehat \theta$ obtained using \SVI. When we train the last few layers, we can generalize this approach and present an algorithm in Section \ref{sec:heuristic}.
% .

\section{Guarantee of monotone VI for last-layer training}\label{sec:guarantee}

We now present guarantees on convergence and recovery of ``ideal'' parameters for the \textit{last-layer training}. In particular, we consider learning of $\theta_L$ when $g_L(X_L,\theta_L)=\eta(X)\theta_L$, where $\eta(X)$ is defined in \eqref{eq:hidden_feature} for last-layer training. In Section \ref{sec:modulus_1}, we can provide an error bound on predicting $\condexp{Y}{X}$ averaged over independent test samples $X$. %, 

Define the set $\mathcal D^{\rm Tr}:=\{(X^{(i)},Y^{(i)})\}_{i=1}^N$ containing $N$ training samples, and assume the training samples are generated using the ``ideal model'' with $\theta=\theta^*$. Let $\Xstar{i,L}$ be the extracted feature, under the ideal model, from previous $L-1$ ``frozen'' layers, following the definition in \eqref{eq:hidden_feature} with the input $X=X^{(i)}$:
\[\eta^*(X^{(i)}):= \Xstar{i,L} :=  \phi_{L-1}\circ g_{L-1} \circ \cdots \phi_{1}\circ g_{1}  (X^{(i)}, \theta^*).\]
For the monotone operator $F(\theta_L)$ in (\ref{eq:operatr}), consider its empirical version
\[
\Femp(\theta_L)= \frac{1}{N} \sum_{i=1}^N {\Xstar{i,L}}^\intercal[\phi_L (\Xstar{i,L} \theta_L)-Y_i].
\]

Let $\widehat{\theta}_L^{(T)}$ denote the learned parameter after $T$ training steps following the \SVI~ iteration \eqref{eq:VI_training} using $\Femp$ above. For a new test sample $X$ with featured mapped under the ideal model with first $L-1$ layers 
\[\Xstar{L} = \eta^*(X),\] we consider the prediction for the test sample $X$ under the learned parameters
\begin{equation}\label{posterior_estimate}
\widehat{Y}(X,\widehat{\theta}_L^{(T)}):=\phi_L(\Xstar{L}\widehat{\theta}_L^{(T)}).
\end{equation}
The prediction  \eqref{posterior_estimate} will be measured against the ideal model prediction $\condexp{Y}{X}=f(X,\theta^*)$. 

We first state a few properties of $F(\theta_L)$ defined in (\ref{eq:operatr}) using $\eta^*(X)$, i.e., the feature mapped using ideal model and identify the form of $\kappa$. 
\begin{lemma}\label{lem1}
Assume $\phi_L: \mathbb R^p \rightarrow \mathbb R^p $ is $K$-Lipschitz continuous and monotone on its domain. For arbitrary test sample $X$,
$F(\theta_L)$ is monotone with modulus $\kappa$ and $K_2$-Lipschitz, where $K_2=K \EE_X\{ \|\eta^*(X)\|^2_2$\}, and
\begin{equation}
    \kappa=\lambda_{\min}(\nabla \phi_L) \lambda_{\min}(\EE_X[{\eta^*(X)}^\intercal \eta^*(X)]). \label{kappa_def}
\end{equation}
Morevoer, $F(\theta^*_L)=0$.

\end{lemma}
The property $F(\theta^*_L)=0$ implies that the ideal last-layer model parameter $\theta^*_L$ is a solution to $\VI$ constructed for the last-layer. In particular, if $\kappa>0$, it is the \textit{unique} solution. As a result, $\theta^*_L$ can be efficiently solved to a high accuracy using $F(\theta_L)$ under appropriated chosen $\gamma$. 
the precise statement is provided in Section \ref{sec:modulus_1}.

\subsection{Case 1: modulus $\kappa > 0$}\label{sec:modulus_1}

Under appropriate step size selection and additional regularity assumptions, we can obtain bounds on the recovered parameters following techniques in \citep{VI_est}, which further enables us to bound the error of model prediction. Note that the guarantee relies on $\kappa$, which may be estimated using $N$ training samples.
\begin{lemma}[Parameter recovery guarantee]\label{thm:alg_guarantee}
    Suppose that there exists $M<\infty$ such that  $ \forall \theta \in \bTheta$, 
    \[
    \EE_{X,Y(\theta)}{\|X_{L}^* Y(\theta)\|_2}\leq M,
    \]
    where $\condexp{Y(\theta)}{X}=\phi_L(X_{L}^*\theta)$. Choose adaptive step sizes $\gamma_t=[\kappa(t+1)]^{-1}$ in \eqref{eq:VI_training}. Then the sequence of estimators $\widehat{\theta}_L^{(T)}$ obeys the error bound 
    \begin{equation}\label{eq:algo_err_bound}
        \EE_{\mathcal D^{\rm Tr}}{\{\|\widehat{\theta}_L^{(T)}-\theta^*_L\|^2_2\} \leq \frac{4M^2}{\kappa^2(T+1)}}.
    \end{equation}
    Above, the expectation is with respect to the randomness of training data, since $\widehat{\theta}_L^{(T)}$ depends on training data $\mathcal D^{\rm Tr}$.
\end{lemma}
One implication of the lemma is that the convergence rate of the algorithm depends on the smallest eigenvalue of the covariance matrix of the extracted features $\theta(X)$; if the extracted features after $L-1$-layers are more uncorrelated with each other, the smallest eigenvalue will be larger. We will show numerically that, in practice, the performance depends on the smallest eigenvalue of the sample covariance matrix of $\theta(X)$ and, thus, depends on the sample size.

We can use the Lemma \ref{thm:alg_guarantee} to bound the prediction error as below. 

\begin{theorem}[Prediction error using strongly monotone $F$]\label{thm:generalization_err}
Given a test sample $X$, the expected prediction error by \eqref{posterior_estimate} is bounded by 
\[
\mathbb E_{(X, \mathcal D^{\rm Tr})}\{\|\widehat{Y}(X, \widehat{\theta}_L^{(T)})-\condexp{Y}{X}\|^2_2\} \leq 
\frac{C_t}{T+1},
\]
where $C_t=4M^2K^2\lambda_{\max}(\mathbb E_X [{X_{L}^*}^\intercal {X_{L}^*}])\cdot\kappa^{-2}$, and the expectation is with respect to the randomness of training samples and test sample, respectively.
\end{theorem}
Note that $p=2$ yields the sum-of-squared error bound on prediction, and $p=\infty$ yields an entry-wise bound.
Additionally, the proof of Theorem \ref{thm:generalization_err} only requires access to an unbiased estimator of $F$, so that the batch size can range from one to $N$, where $N$ is the size of the training data.

\begin{remark}[When does $\kappa>0$?]
Recall that the modulus $\kappa$ is defined as \eqref{kappa_def}. To have $\lambda_{\min}(\mathbb{E}_X[X^{* \intercal}_L\Xstar{L}])$ is bounded away from zero, we thus are concerned with the minimum eigenvalue of the gradient of $\phi$ acting on its inputs. Note that the Jacobian matrix is diagonal when $\phi_L$ is an element-wise function on its vector inputs. In this case, we only need the element-wise activation function to be continuously differentiable with positive derivatives; for instance, the sigmoid function, for any $y\in \R^p$
\[
\lambda_{\min}[\nabla \phi_L]|_{y}=\min_{i=1, \ldots, p} \phi_L(y_i)(1-\phi_L(y_i)),
\]
which is bounded away from zero. However, the ReLu function does not satisfy this requirement.
\end{remark}

\subsection{Case 2: modulus $\kappa = 0$}\label{sec:modulus_2}

In practice, we may encounter cases where the operator $F$ is only monotone but not strongly monotone. For instance, when $\phi$ is the softmax function that applied element-wise to $\eta^*(X)\theta_L \in \mathbb R^{n\times F}$, the Jacobian matrix is block-diagonal. For any vector $z\in \R^F$, $\nabla \phi(z)=\text{diag}(\phi(z))-\phi(z)\phi(z)^T$, which satisfies $\nabla \phi(z)\boldsymbol 1=\boldsymbol 0$ for any $z$ \citep[Proposition 2]{softmax}. Therefore, the minimum eigenvalue of the gradient matrix of $\phi$ is always zero, and hence $\kappa=0$. 

In this case, note that the solution of $\VI$ needs not be unique: a solution $\bar{\theta}_L$ that satisfies the condition $\langle F(\bar{\theta}_L), \theta-\bar{\theta}_L\rangle \geq 0, \ \forall \theta \in \bTheta$ may not correspond to the ideal $\thetastar{L}$ (so there is no identifiability). Nevertheless, we directly approximate the zero of $F$ by using the operator extrapolation method (OE) in \citep{kotsalis2022simple}.
We then have an $\ell_p$ bound on prediction error:
\begin{theorem}[Prediction error using monotone $F$]\label{thm:generalization_err_nostrong}
Suppose we run the OE algorithm \citep{kotsalis2022simple} for $T$ iterations with $\lambda_t=1$, $\gamma_t=1/(4K_2)$, where $K_2$ is the Lipschitz constant of $F$. Let $R$ be uniformly chosen from $\{2,3,\ldots,T\}$. Then
\[
    \EE_{\mathcal D^{\rm Tr}} \{\| \EE_X\{\sigma_{\min} (\eta^*(X)^{\intercal})[\widehat{Y}(X, \widehat{\theta}_L^{(R)})-\condexp{Y}{X}]\}\|_2 \}\leq 
    \frac{C_t^{''}}{\sqrt T},
\]
where $\sigma_{\min} (\cdot)$ denotes the minimum singular value of its input matrix and the constant $C_t^{''}=3\sigma+12K_2(2\|\thetastar{L}\|_2^2+2\sigma^2/L^2)^{1/2}$, in which $\sigma^2=\EE_{\mathcal D^{\rm Tr}}[(\widehat{F}(\theta_L)-F(\theta_L))^2]$ is the variance of the unbiased estimator $\widehat{F}(\theta_L)$.
\end{theorem}

The convergence rate in Theorem \ref{thm:generalization_err_nostrong} is unaffected by the batch size, which only serves to reduce the variance. In addition, Theorem \ref{thm:generalization_err_nostrong} requires $R$ be uniformly chosen from $\{2,3,\ldots,T\}$, so that the theoretical guarantee holds at a random training epoch. In theory, this assumption is necessary to ensure a decrease of the norm of the monotone operator (\citep{kotsalis2022simple}, Eq. (3.20)). In practice, we observed that the epoch that leads to the highest validation accuracy might not occur at the end of $T$ training epochs, so this assumption is reasonable based on empirical evidence.

\subsection{Special cases: Monotone operator vector field becomes gradient} \label{sec:equivalence}

In practice, neural networks are commonly trained via empirical loss minimization, in contrast to solving the monotone VI approach here. Nevertheless, it can be shown that if not solving minimizing the mean-square-error as in \eqref{LS}, but instead using the cross-entropy loss when $\phi$ is either the sigmoid function or the softmax function, the monotone operator $F$ defined in \eqref{eq:operatr} will corresponds to the expectation of the gradient of the loss with respect to parameters. As a result, the two approaches (SGD and \SVI) yield the same last-layer training algorithm. 

Consider a pair of input $X \in \R^{d}$ and output $Y \in \{0,\ldots, k\}$. Let $e_Y\in \{0,1\}^{k+1}$ be the one-hot encoding of $Y$.
Now, the cross-entropy loss $\Lc(\theta_L)$ is 
\begin{align}
    \Lc(\theta_L)&=-Y\log (\phi_L(\eta^*(X)\theta_L))-(1-Y)\log (1-\phi_L(\eta^*(X)\theta_L)), && \text{$Y \in \{0,1\}$}. \label{eq:binaryCE}\\
    \Lc(\theta_L)&=-e_{Y}^T \log(\phi_L(\eta^*(X)\theta_L)), \quad \text{$Y \in \{0,\ldots,k\}, k>1$}. \label{eq:multiCE}
\end{align}
In binary classification \eqref{eq:binaryCE}, $\phi_L(x)=1/(1+e^{-x})$ is the element-wise sigmoid function. In multi-class classification \eqref{eq:multiCE} $\phi_L(x)= [e^{x_1}/\sum_j e^{x_j}, \ldots, e^{x_{k+1}}/\sum_j e^{x_j}]$ is the softmax function. We now have the following proposition.

\begin{proposition}[The equivalence between SVI and parameter gradient] \label{prop:equivalence} 
Under the setup for \eqref{eq:binaryCE} or \eqref{eq:multiCE}, we have that for any $\theta_L \in \bTheta$ 
\[
\EE_{X,Y}[\nabla_{\theta_L} \Lc(\theta_L)]=F(\theta_L),
\]
where the monotone operator $F(\theta_L)$ is defined in \eqref{eq:operatr}, and the expectation is with respect to randomness in both input $X$ and output $Y$.
\end{proposition}

\section{Extension: Training  multiple last layers}\label{sec:heuristic}

Now we look into generalizing the last one-layer training by considering training the last few layers of a network, from layer $L'$ to $L$, $1\leq L'\leq L$, and typically the number of trained layers $L-L'\geq 1$ is small. We will show how to construct the vector field based on \eqref{eq:operatr}. Specifically, we aim to design $F_l(\theta_l)$, $l = L', \ldots, L$, for parameters $\theta_l$ in layer $l$ of the neural network, similar to the scheme before, and use these $F_l(\theta_l)$ as update directions for the parameters $\{\theta_l\}_{l=L'}^L$ in training by \eqref{eq:VI_training}.

\begin{figure}[!b]
    \centering
\includegraphics[width=.8\linewidth]{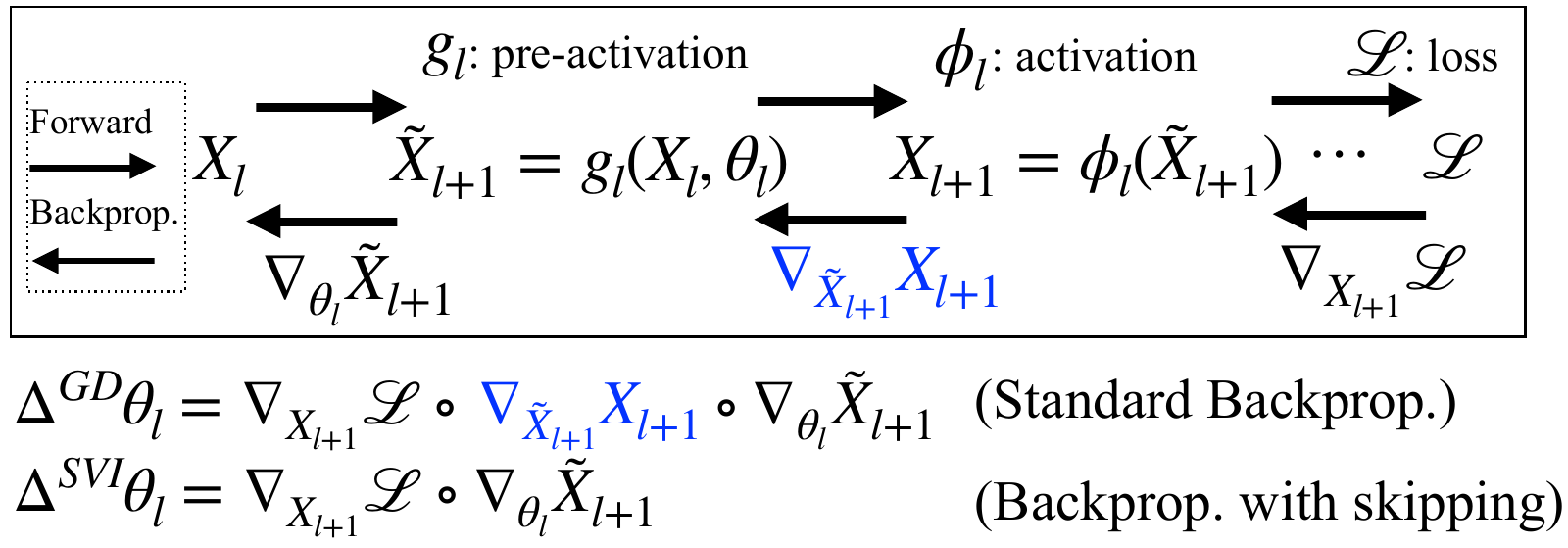}
    \caption{Gradient descent (GD) vs. \texttt{SVI} for fine-tuning an arbitrary layer $l$, $L'\leq l\leq L$, via {\it backward unrolling scheme}: the difference lies in the skipping of differentiation of the activation $\phi$ with respect to pre-activation values. The notation $\circ$ denotes function composition.}
    \label{illustrate}
\end{figure}

However, there are several difficulties in directly extending $F(\theta)$ in \eqref{eq:operatr} to train multi-layer neural networks. First; the definition assumes that the pre-activation mapping $\eta(X)\theta$ is linear with respect to $\theta$. This linearity assumption does not hold for many types of network layers. Second, more importantly, the monotone operator \eqref{eq:operatr} is defined using $Y$, the observation of the response variable. However, such observation is unavailable when we define $F_l(\theta_l)$ for intermediate layers, where $Y_l$ would denote the ``true'' value to be approximated by outputs from the $l$-th hidden layer. Third, how to generalize \eqref{eq:operatr} for different loss objectives is unclear. Thus, we will build a heuristic proven to be good in practice.

\subsection{Revisit single-layer training}

To motivate the monotone VI-inspired scheme designed for multiple last-layer training, we first present a mathematically equivalent view of $F(\theta)$ in \eqref{eq:operatr} for single-layer training. Consider for the last layer of the neural network, the linear pre-activation mapping for input $\eta(X)$ (feature mapped from the previous layers), with parameter $\theta$  
\[\tilde{X}(\theta)=\eta(X)\theta.\] Note that the $\tilde X(\theta)$ does not need to depend on $\theta$ linearly, although in most neural networks the dependence is linear. The corresponding prediction is given by $\hat{Y}=\phi (\tilde{X})$, and consider the mean-squared-error (MSE) loss:
\begin{equation}\label{L_def}
\mathcal{L}(\hat{Y},Y)=\frac 1 2\|\hat{Y}-Y\|^2_2. 
\end{equation}
Recall, by our construction, the vector field can also be written as 
\begin{align}
    F(\theta):
    & =\EEXY{\etaT{X}[\hat Y-Y]} \nonumber \\ 
    & = \EEXY{\nabla_{\hat{Y}} \mathcal{L} \circ \nabla_{\theta} \tilde{X}(\theta)}, \label{equiv}
\end{align}
where the last equality is due to $\nabla_{\hat{Y}}\mathcal{L} = (\hat Y - Y)$, $\nabla_{\theta} \tilde{X}(\theta) = \etaT{X}$.
Thus, this observation says that the monotone operator vector field can be constructed as the product of two terms: where the first term in \eqref{equiv}, $\nabla_{\hat{Y}} \mathcal{L}$, the gradient of the loss objective with respect to the network prediction $\hat{Y}$ says the sensitivity of the loss regarding the network prediction $\hat Y$, and the second term $\nabla_{\theta} \tilde{X}$ is the gradient of the \textit{pre-activation} mapping $\tilde{X}$ with respect to the parameter $\theta$, which due to linearity is the input of the network. 

\subsection{{\it Backward unrolling scheme} for multi-layer training}

Based on this observation from \eqref{equiv}, we propose \SVI{} for multiple layer training based on a {\it backward unrolling scheme} as in Algorithm \ref{alg:SVI} and illustrated in Figure \ref{illustrate}. For notational simplicity, we present the scheme for a single data point $(X, Y)$, and the actual algorithm can use the average over multiple samples:  
\begin{itemize}
\item For a $L$-layer network, to tune the last layer, since 
\[
\mathcal L = \frac 1 2 \|\underbrace{\phi_L \circ g_L (X_{L})}_{\hat Y} -Y\|_2^2,
\] we can directly use \eqref{equiv} by treating the input of last layer as $\eta(X) =X_L$: the mapping of original input $X$ into features using all first $L-1$ layers, and $\hat Y$ is the prediction by forward predict passing the input $\eta(X)$ through the last layer. 

\item Now, to train the $L-1$ layer's weights while fixing the $L$th layer, we note that the loss can be written as 
\[
\mathcal L = \frac 1 2 \|\phi_L \circ g_L \circ \underbrace{\phi_{L-1} \circ g_{L-1} (X_{L-1})}_{X_L} -Y\|_2^2
\]
where we can use $\eta(X) = X_{L-1}$ as the features learned by the previous $L-2$ layers, and use the sensitivity of the loss with respect to the output of the $(L-1)$th layer, $X_L$, in the construction of the monotone operator vector field, $\nabla_{X_L} \mathcal{L} = \nabla_z \left( \frac 1 2 \|\phi_L \circ g_L (z)\|^2_2\right)$.

\item In general, to train layers from $l$, $L' \leq l \leq L$, we notice that
\[
\mathcal L = \frac 1 2 \|\phi_L \circ g_L \cdots \circ \phi_{l} \circ g_{l} (X_{l}) -Y\|_2^2,
\]
using the idea to construct the monotone VI vector field by using $\eta(X) = X_L$, we can use \eqref{equiv} with $\eta(X) = X_{l}$, and 
\begin{equation}\label{loss_grad}
    \nabla_{X_{l+1}}\mathcal L = \nabla_z \left(\frac 1 2 \|\phi_L \circ g_L \cdots \circ \phi_{l+1} \circ g_{l+1}  (z)\|_2^2\right)
\end{equation}
\end{itemize} 
Note that compared to the commonly used gradient $\nabla_{\theta} \mathcal{L}$ in backpropagation, Algorithm \ref{alg:SVI} only differs by \textit{skipping} the derivative of the point-wise non-linearity $\phi$ with respect to its input. Therefore, \SVI{} has a similar computational cost against gradient-based methods. In practice, this skipping leads to different dynamics when updating neurons in the network; see Remark \ref{remark:dynamic} for details.

\subsection{Practical considerations}

The approach taken in Algorithm \ref{alg:SVI} has two main benefits. First, it addresses the challenges of extending monotone VI to multi-layer training. More precisely, the Algorithm applies to arbitrary forms of network layers $g_l(X_l,\theta_l)$, nonlinear activations $\phi_l$, and the loss function $\mathcal{L}$. It also requires no observation of responses from neurons in hidden layers. 
Second, it is easy to implement by leveraging automatic differentiation \citep{paszke2017automatic}. Specifically, we implement the skipping idea via backpropagating the layer-wise surrogate loss $\widetilde{\mathcal{L}}_l$ with respect to $\theta_l$. Note that this loss $\widetilde{\mathcal{L}}_l$ is simple to compute: the quantity $\tilde{X}_{l+1}$ is available during the forward pass on training data $X$, and gradients $\nabla_{X_{l+1}} \mathcal{L}$ are available upon backpropagating the original loss $\mathcal{L}$ with respect to outputs of each layer $l$. 

\begin{remark}[Effect on training dynamics] \label{remark:dynamic}
We remark on a key difference between parameter update in SGD and in \SVI, which ultimately affects training dynamics. Suppose the activation function is ReLU. It is well-known that SGD does not update weights of \textit{inactive neurons} (i.e., ReLU$(x)=0$) because the gradient of ReLU with respect to them is zero. However, \SVI{} does not discriminate between active and inactive neurons as it \textit{skips} this derivative computation. Thus, one can expect that \SVI{} results in a more significant weight update than SGD, which experimentally seems to speed up the initial model convergence. We illustrate this phenomenon in Figure \ref{fig:dynamics}. 
\end{remark}

\begin{remark}[Implementation caveats]
We have two remarks that help avoid incorrect implementation when \SVI{} is implemented in existing software (e.g., PyTorch \cite{NEURIPS2019_9015}). 
First, before obtaining $\{\nabla_{X_{l+1}} \mathcal{L}\}_{l=1}^L$ in line 5, model parameters $\{\theta_l\}$ should not require gradients. Otherwise, the backpropagation of the original loss to obtain $\{\nabla_{X_{l+1}} \mathcal{L}\}_{l=1}^L$ would also store gradients of $\mathcal{L}$ with respect to $\theta$ before the \SVI{} updates. 
Second, after obtaining $F_l^{(i)}(\theta_l)$ in line 8, $\theta_l$ should not require gradient (until the next for loop in line 6 is called). Otherwise, because $\tilde{X}_{l+2}$ as the output by the $(l+1)$-th layer implicitly depends on $\theta_l$, backpropagating the surrogate loss $\widetilde{\mathcal{L}}$ with respect to model parameters would incorrectly accumulate update directions for $\theta_l$.
\end{remark}

\begin{remark}[Extension to GNN] Our approach can also be extended for special neural network architectures, such as GNN and CNN. Below we remark on how to extend to GNN. Suppose we have an undirected and connected graph $\G=(\V,\E, W)$, where $\V$ is a finite set of $n$ vertices, $\E$ is a set of edges, and $W\in\R^{n\times n}$ is a weighted adjacency matrix that encodes node connections. Let $I_n$ denote an identity matrix of size $n$. Let $D$ be the degree matrix of $W$ and $L_g=I_{n}-D^{-1 / 2} W D^{-1 / 2}$ be the normalized graph Laplacian, which has the eigen-decomposition $L_g=U\Lambda U^T$. For a graph signal $X \in \R^{n \times d}$ with $d$ input channels, it is then filtered via pre-activation function in the form of $g_l(L_g, \theta_l)$ which acts on $L_g$ with channel-mixing parameters $\theta_l \in \R^{d\times k}$ for $k$ output channels. Thus, the filtered signal in the $l$-th layer is given by $X'=g_l(L_g, \theta_l)X$. 
It is common that $g_l(L_g, \theta_l)X=\eta(X)\theta_l$, where $\eta(X)=\sum_{r=1}^R h_r(L_g)X$ 
is the sum of $R$ fixed graph filters determined by graph Laplacian $L_g$ \citep{GCN,chebnet,GraphSAGE}.  
\end{remark}

\begin{algorithm}[!t]
\cprotect \caption{Stochastic variational inequality (\SVI) with backward unrolling scheme}
\label{alg:SVI}
\begin{algorithmic}[1]
\REQUIRE{ 
{
(a) Training data $\{(X^{(i)},Y^{(i)})\}_{i=1}^N$,
(b) $L$-layer network $f(X,\theta)=\{\phi_l\circ g_l(X_l,\theta_l)\}_{l=1}^L$,
(c) Loss function $\mathcal{L}$, 
(d) Learning rate $\gamma > 0$, (e) desired $L' \in [1, L-1]$
(f) Number of steps $S$
}}
\ENSURE{Trained model parameters $\hat{\theta}=\{\hat{\theta}_l\}_{l=L'}^L$}
\STATE Initialize the network with $\hat{\theta}$
\FOR{update step $s=1,\ldots,S$}
\FOR{each training sample $i=1,\ldots,N$}
\STATE Store $\{(\tilde{X}_{l+1}, X_{l+1})\}_{l=L'}^L$  with input $X^{(i)}$:  
$\tilde{X}_{l+1}=g_l(X_l,\hat{\theta}_l)$ and $X_{l+1} = \phi_l(\tilde{X}_{l+1})$
\STATE Obtain $\{\nabla_{X_{l+1}} \mathcal{L}\}_{l=L'}^L$  defined in \eqref{loss_grad}
\FOR{Layer $l=L',\ldots,L$}
\STATE Compute the surrogate loss $\widetilde{\mathcal{L}}_l= \langle \nabla_{X_{l+1}} \mathcal{L},\tilde{X}_{l+1}\rangle$
\STATE Obtain $F^{(i)}_l(\theta_l)=\nabla_{X_{l+1}} \mathcal{L} \circ \nabla_{\theta_l} \tilde{X}_{l+1}=\nabla_{\theta_l} \widetilde{\mathcal{L}}_l$
\ENDFOR
\ENDFOR
\STATE Update $\hat{\theta}_l = \hat{\theta}_l- \gamma \left(\frac{1}{N}\sum_{i=1}^N F^{(i)}_l(\theta_l)|_{\theta_l = \hat{\theta}_l}\right)$ for $ l=L',\ldots,L$
\ENDFOR
\end{algorithmic}
\end{algorithm}

\section{Experiments}\label{sec:expr_main}

We test and compare \SVI{} in Algorithm \ref{alg:SVI} with SGD on several synthetic and real-data experiments, where the networks vary in width and depth. We aim to demonstrate the benefits of \SVI{} in terms of reaching faster convergence and competitive/better final performance\footnote{The code is available at \url{https://github.com/hamrel-cxu/SVI-NN-training}.}.

\subsection{Result summary}

Below is a summary of the experiments we performed in this section.

\vspace{0.1in}
\noindent \textit{One-layer networks} (Sec. \ref{sec:one_layer}): We provided theoretical guarantees in Section \ref{sec:guarantee}. When data are generated from a non-convex probit model, \SVI{} outperforms SGD with smaller test losses and higher prediction accuracies.

\vspace{0.1in}
\noindent \textit{Two-layer networks} (Sec. \ref{sec:two_layer}): \SVI{} reaches smaller losses and errors than SGD throughout training epochs. Specifically, we observe performance gains on graphs with varying sizes (i.e., the number of graph nodes ranges from 15 to 600).

\vspace{0.1in}
\noindent \textit{Three-layer networks} (Sec. \ref{sec:solar} and \ref{sec:traffic}): On real solar and traffic data, \SVI{} almost always reaches smaller classification error and higher weighted $F_1$ scores than SGD, in addition to faster initial convergence by \SVI{}. 

\vspace{0.1in}
\noindent \textit{Networks with more than three layers, using CNN and GNN} (Sec. \ref{sec:ogb} and \ref{sec:img}): We compare \SVI{} and SGD on node classification for large graphs and image classification. The node classification task uses one large-scale realistic node classification dataset from the Open Graph Benchmark \citep{hu2020ogb}. The image classification uses the MNIST \citep{MNIST} and CIFAR-10 \citep{CIFAR10} datasets. In particular, the results show that \SVI{} yields improved efficiency during the initial stages of training and reaches competitive overall training performances.

\subsection{Setup and comparison metrics}

\noindent \textit{Setup.} All implementation are done using \texttt{PyTorch} \citep{NEURIPS2019_9015} and \texttt{PyTorch Geometric} \citep{Fey/Lenssen/2019} (for GNN). To ensure a fair comparison, we carefully describe the experiment setup. In particular, the following inputs are \textit{identical} to both \SVI{} and SGD in each experiment. 
\begin{itemize}[noitemsep,topsep=0em]
    \item Data: (a) the size of training and test data (b) batch (batch size and samples in mini-batches).
    \item Model: (a) architecture (e.g., layer choice, activation function, hidden neurons) (b) loss function.
    \item Training regime: (a) parameter initialization (b) hyperparameters for backpropagation (e.g., learning rate) (c) total number of epochs.
\end{itemize}
Thus, all except the update directions for parameters are kept the same for a fair comparison: our \SVI{} uses $F_l(\theta_l)$ in Algorithm \ref{alg:SVI} and SGD uses the gradient of the loss with respect to parameters.

\vspace{0.1in}
\noindent \textit{Comparison metrics.} Consider inputs $X\in \R^{n\times d}$, where $n$ is the number of graph nodes (when the data are vectors, we have $n=1$) and $d$ is the feature dimension per node. Let the true (or predicted) model be $\EE[Y|X,\theta]\in \R^{n\times k}$ (or $\EE[Y|X,\widehat \theta]$), where $k$ is the output dimension per node. Given $N$ samples $\{(X^{(i)},Y^{(i)})\}_{i=1}^N$, we consider the following metrics.
{\small \begin{align}
     \text{MSE loss}&= \frac 1N \sum_{i=1}^N \|\EE[Y^{(i)}|X^{(i)},\widehat \theta]-Y^{(i)}\|_2 \label{eq:MSE_loss} \\
     \text{Cross-entropy loss}&= \frac 1N \sum_{i=1}^N \langle Y^{(i)}, -\log(\EE[Y^{(i)}|X^{(i)},\widehat \theta]) \rangle \label{eq:CE_loss}\\
     \text{Classification error}&=\frac{1}{n N k}\sum_{i=1}^N\sum_{j=1}^n \sum_{f=1}^k \textbf{1}(Y^{(i)}_{j,f}\neq \hat Y^{(i)}_{j,f}) \label{eq:classification_error} \\
     \ell_2\text{ model recovery error}&= \frac 1N \sum_{i=1}^N\|\EE[Y^{(i)}|X^{(i)},\widehat \theta]-\EE[Y^{(i)}|X^{(i)},\theta]\|_2 \label{eq:l_2_model_err}.
\end{align}}

In addition, all results are averaged over three random trials, where networks are re-initialized in each trial. In simulation, we also redraw training samples in each trial. We show standard errors in tables as brackets and plots as error bars. 

Notation-wise, $X\sim \mathcal N(a,b)$ means the random variable $X$ follows a normal distribution with mean $a$ and variance $b^2$; $N$ (resp. $N_1$) denotes the size of training (resp. test) sample, $\textsf{lr}$ denotes the learning rate, $B$ denotes the batch size, and $E$ denotes training epochs.

% Probit
\begin{figure}[!t]
    \centering
    \begin{minipage}{0.49\textwidth}
        \includegraphics[width=\linewidth]{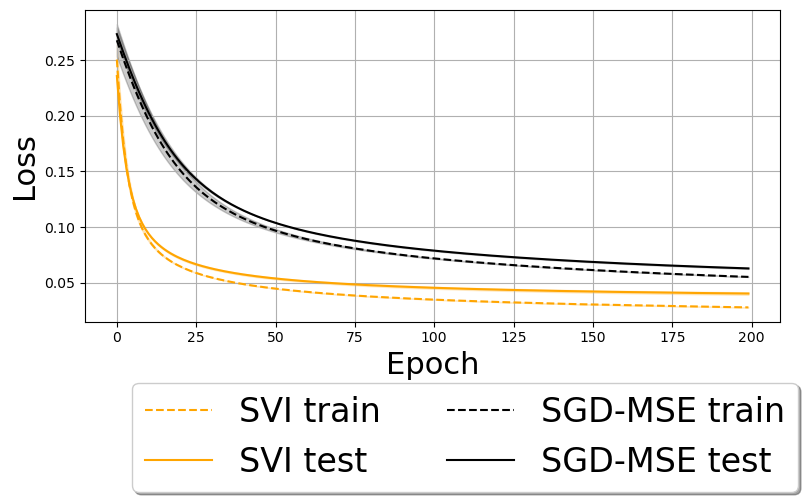}
        \subcaption{Losses under MSE objective}
    \end{minipage}
    \begin{minipage}{0.49\textwidth}
        \includegraphics[width=\linewidth]{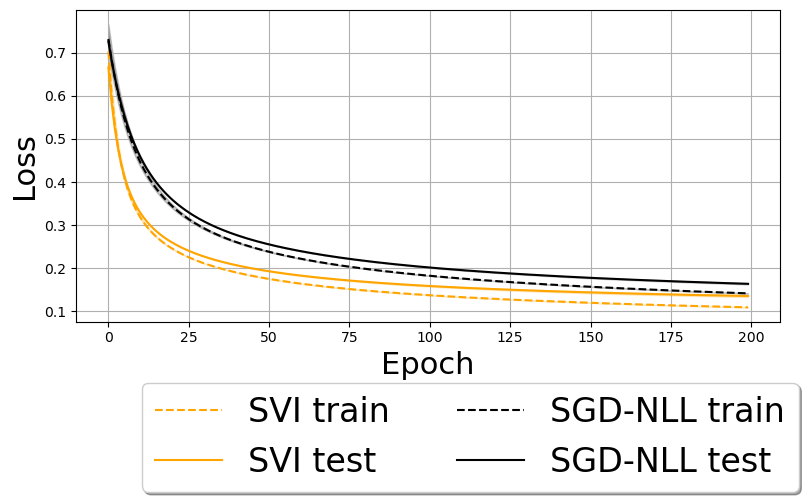}
        \subcaption{Losses under NLL objective}
    \end{minipage}
    \begin{minipage}{0.55\textwidth}
        \includegraphics[width=\linewidth]{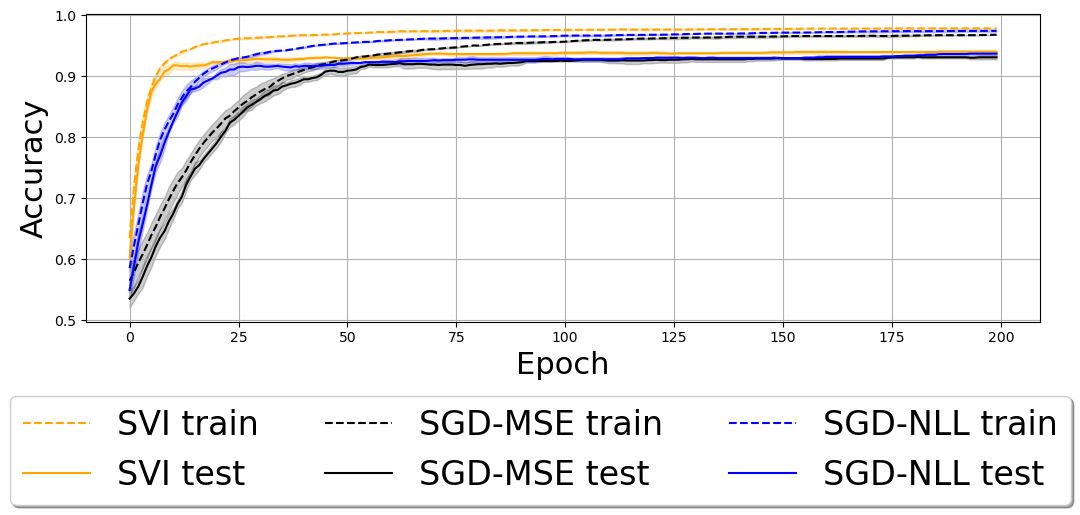}
        \subcaption{Accuracies on the same test data}
    \end{minipage}
    \cprotect \caption{One-layer fully-connected network. In (a) and (b), we plot training and test losses for both SGD (black) and \SVI{} (orange). The loss choice for SGD is the mean-squared error (MSE) in (a) and the negative log-likelihood (NLL) in (b). The \SVI{} update for one-layer training is based on sample version of \eqref{eq:operatr}, which does not depend on the loss objective; thus, we compute the MSE or NLL losses for \SVI{} for comparison. In (c), we show the binary classification accuracies by \SVI{} in orange, SGD by MSE (SGD-MSE) in black, and SGD by NLL (SGD-NLL) in blue.}
    \label{fig:one/last-layer}
\end{figure}

\subsection{Synthetic data experiments}

\begin{figure}[!t]
    \centering
    \includegraphics[width=\linewidth]{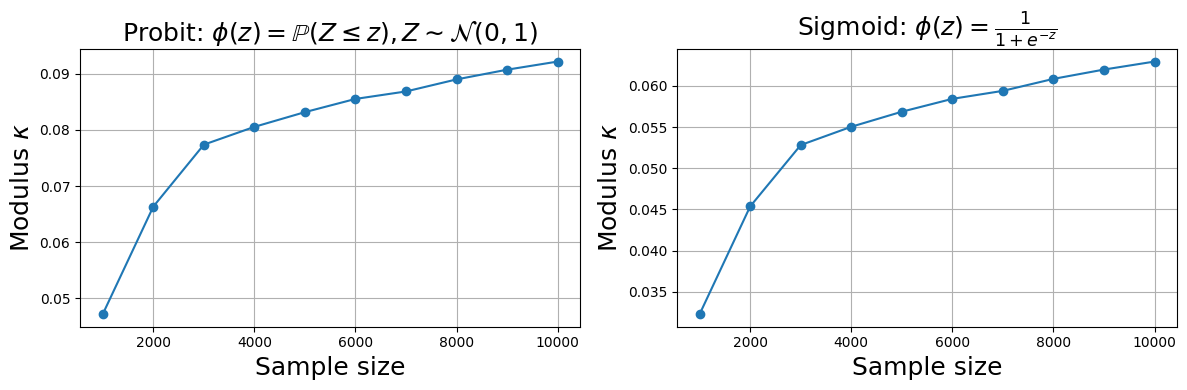}
    \caption{Modulus $\kappa$ evaluated on training data for the one-layer network training in Section \ref{sec:one_layer}. We vary the activation function $\phi$ and plot the corresponding $\kappa$ values over increasing number of training samples.}
    \label{fig:kappa_val}
\end{figure}

\vspace{0.1in}

\subsubsection{One-layer network training}\label{sec:one_layer}
We first consider data generated from a one-layer probit model, which leads to non-convex objectives. Specifically, for each $i\geq 1$, let $X^{(i)} \in \R^p, Y^{(i)} \in \{0,1\}$ and 
\[\condexp{Y^{(i)}}{X^{(i)}}=\Phi(\beta^T X^{(i)}+b),\]
where $\Phi(z)=\PP(Z\leq z)$ for $Z\sim \mathcal N(0,1)$.
Let the subscript $j$ denotes the $j$-th entry. We then sample $X^{(i)}_j$ $\overset{i.i.d.}{\sim} \mathcal N(0.05,1)$, $\beta_j \overset{i.i.d.}{\sim} \mathcal N(-0.05,1)$, and $b \sim \mathcal N(-0.1,1)$. 
We let $N=2000$, $N_1=500$, and use a fully-connected one-layer network. We further let $B=200$ and $E=200$ and use $\textsf{lr}=0.005$.

To train \SVI{} and SGD, we use both the MSE objective and the negative log-likelihood (NLL) objectives. Note that in this one-layer setting, \SVI{} provides monotone operators as parameter update directions that are independent of the training objective (recall the definition in \eqref{eq:operatr}). Figure \ref{fig:one/last-layer} shows the training and test performances by both methods. On test data, \SVI{} consistently yields smaller losses and higher prediction accuracies than SGD under both MSE and NLL objectives during training. This shows the effectiveness of \SVI{} in cases with theoretical guarantees. Figure \ref{fig:kappa_val} additionally visualizes the modulus $\kappa$ as a function of the sample size, where we see $\kappa$ is always positive (so the operator is strongly monotone) and increases as sample sizes grow.

\vspace{0.1in}

\begin{table}[!t]
\caption{Two-layer GNN model on random graphs with an increasing number of graph nodes. We show the training and test $\ell_2$ error as defined in \eqref{eq:l_2_model_err} along the training epochs. Entries in brackets indicate standard deviation over 3 independent initialization of model parameters.}
\label{tab:more_nodes}
\begin{minipage}{0.49\textwidth}
    \centering
    \subcaption{Epoch 50}
    \resizebox{\linewidth}{!}{
    \begin{tabular}{P{1.6cm}|P{1.6cm}P{1.6cm}P{1.6cm}P{1.6cm}}
    \toprule
    \# nodes & SGD train & SVI train & SGD test & SVI test \\
    \midrule
    15  &  0.110 (2.01e-2) &  0.104 (1.72e-2) &  0.107 (1.63e-2) &  0.101 (1.34e-2) \\
    40  &  0.102 (1.60e-2) &  0.096 (1.39e-2) &  0.101 (1.60e-2) &  0.095 (1.38e-2) \\
    100 &  0.092 (1.27e-2) &  0.087 (1.13e-2) &  0.093 (1.37e-2) &  0.088 (1.22e-2) \\
    300 &  0.080 (9.48e-3) &  0.077 (8.80e-3) &  0.081 (1.04e-2) &  0.077 (9.68e-3) \\
    600 &  0.073 (7.63e-3) &  0.070 (7.34e-3) &  0.074 (8.32e-3) &  0.071 (7.98e-3) \\
    \bottomrule
    \end{tabular}}
\end{minipage}
\begin{minipage}{0.49\textwidth}
    \centering
    \subcaption{Epoch 100}
    \resizebox{\linewidth}{!}{
    \begin{tabular}{P{1.6cm}|P{1.6cm}P{1.6cm}P{1.6cm}P{1.6cm}}
    \toprule
    \# nodes & SGD train & SVI train & SGD test & SVI test \\
    \midrule
    15  &  0.099 (1.79e-2) &  0.089 (1.21e-2) &  0.097 (1.46e-2) &  0.087 (9.09e-3) \\
    40  &  0.092 (1.45e-2) &  0.083 (1.04e-2) &  0.092 (1.47e-2) &  0.083 (1.05e-2) \\
    100 &  0.085 (1.18e-2) &  0.078 (9.29e-3) &  0.086 (1.30e-2) &  0.079 (1.03e-2) \\
    300 &  0.076 (9.09e-3) &  0.070 (7.86e-3) &  0.077 (1.01e-2) &  0.071 (8.78e-3) \\
    600 &  0.070 (7.43e-3) &  0.064 (6.86e-3) &  0.070 (8.14e-3) &  0.065 (7.54e-3) \\
    \bottomrule
    \end{tabular}
    }
\end{minipage}
\begin{minipage}{\textwidth}
    \centering
    \subcaption{Epoch 200}
    \resizebox{0.49\textwidth}{!}{
    \begin{tabular}{P{1.6cm}|P{1.6cm}P{1.6cm}P{1.6cm}P{1.6cm}}
    \toprule
    \# nodes & SGD train & SVI train & SGD test & SVI test \\
    \midrule
    15  &  0.084 (1.43e-2) &  0.073 (5.81e-3) &  0.082 (1.17e-2) &  0.071 (4.11e-3) \\
    40  &  0.079 (1.21e-2) &  0.069 (5.76e-3) &  0.079 (1.26e-2) &  0.068 (6.20e-3) \\
    100 &  0.075 (1.06e-2) &  0.066 (6.55e-3) &  0.077 (1.19e-2) &  0.067 (7.75e-3) \\
    300 &  0.069 (8.61e-3) &  0.060 (6.50e-3) &  0.070 (9.67e-3) &  0.061 (7.49e-3) \\
    600 &  0.064 (7.23e-3) &  0.056 (6.05e-3) &  0.065 (7.97e-3) &  0.056 (6.80e-3) \\
    \bottomrule
    \end{tabular}
    }
\end{minipage}
\end{table}

% GNN neuron dynamics
\begin{figure}[!b]
    \centering
    \includegraphics[width=\textwidth]{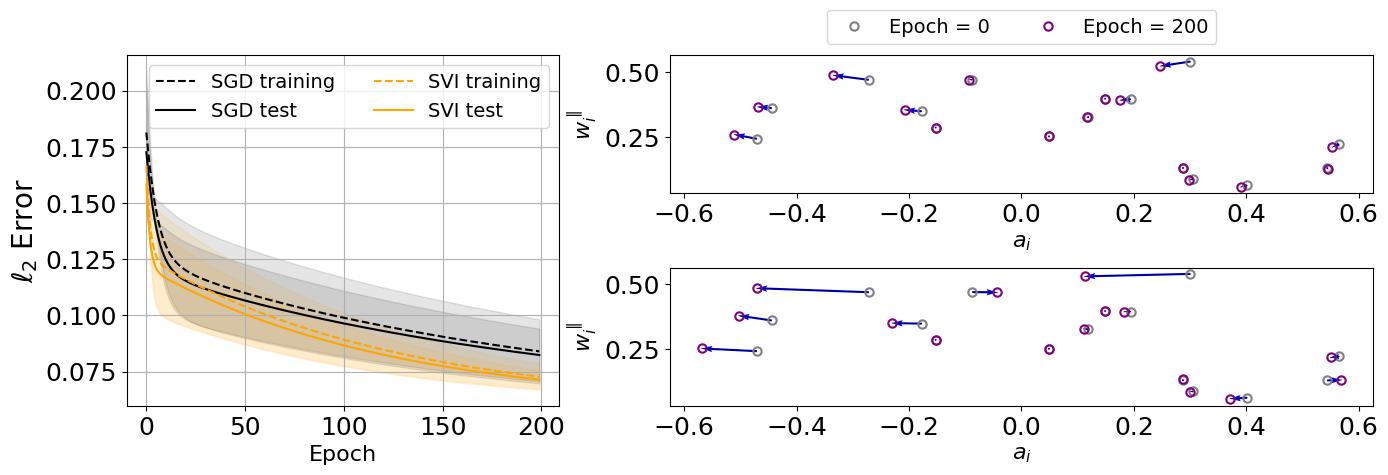}
\cprotect \caption{Two-layer GNN prediction error and neuron dynamics visualization. Left: $\ell_2$ error in \eqref{eq:l_2_model_err}. Right: visualization of the training dynamics by SGD (top) and \SVI{} (bottom).}
\label{fig:dynamics}
\end{figure}

\subsubsection{Two-layer GNN training}\label{sec:two_layer}

Table \ref{tab:more_nodes} shows that \SVI \ consistently reaches smaller $\ell_2$ prediction error for the two-layer GNN, even when the graph is large (600 nodes). This shows the robustness of the proposed \SVI{}, even in this setting beyond our theoretical analyses.
To better understand how \SVI \ update parameters, Fig. \ref{fig:dynamics} zooms in the dynamics for 16 neurons (right figure) and shows the corresponding $\ell_2$ model recovery test error (left figure). Regarding the right figure, we plot the norm of first-layer neuron weights, where the norm is defined in terms of the inner product with initial weights, against the second-layer neuron weights, which are scalars because $k=1$. One circle represents one neuron, with arrows representing the direction of moving along the initial weights. We then connect the initial and final dots to indicate the displacement of neurons. The same visualization techniques are used in \citep{pellegrini2020analytic}. We observe that \SVI \ displaces neurons further after 200 epochs, which is anticipated in Remark \ref{remark:dynamic}. In this case, it can be beneficial due to the faster error convergence.

% Real-data Solar
\begin{table}[!t]
\centering
\caption{Solar ramping event detection under varying sizes of the GNN. Entries in brackets indicate standard deviation over three independent initializations of model parameters.}
\label{tab:solar_tab}
\begin{minipage}{0.49\textwidth}
    \centering
    \subcaption{MSE loss}
    \resizebox{\linewidth}{!}{
    \def\arraystretch{1.25}%
    
    \begin{tabular}{P{1.5cm}|P{1.45cm}P{1.45cm}P{1.45cm}P{1.45cm}}
    \toprule
    {\# hidden neurons} & SGD Training & \SVI{} Training & SGD Test & \SVI{} Test \\
    \midrule
    32 & 0.224 (4.67e-3) & 0.200 (9.95e-4) & 0.223 (2.40e-3) & 0.204 (1.30e-3) \\
    64 & 0.213 (7.66e-4) & 0.191 (8.49e-4) & 0.211 (1.67e-3) & 0.195 (1.12e-3) \\
    128 & 0.219 (1.79e-3) & 0.195 (6.22e-4) & 0.217 (1.13e-3) & 0.199 (5.14e-4) \\
    \bottomrule
    \end{tabular}
    }
\end{minipage}
\begin{minipage}{0.49\textwidth}
    \centering
    \subcaption{Classification error}
    \resizebox{\linewidth}{!}{
    \def\arraystretch{1.25}%
    \begin{tabular}{P{1.5cm}|P{1.45cm}P{1.45cm}P{1.45cm}P{1.45cm}}
    \toprule
    {\# hidden neurons} & SGD Training & \SVI{} Training & SGD Test & \SVI{} Test \\
    \midrule
    32 & 0.297 (5.88e-3) & 0.275 (4.54e-3) & 0.333 (1.28e-2) & 0.296 (1.13e-2) \\
    64 & 0.283 (4.33e-3) & 0.263 (1.59e-3) & 0.308 (1.04e-2) & 0.272 (1.15e-3) \\
    128 & 0.295 (3.15e-3) & 0.267 (2.43e-3) & 0.328 (6.13e-3) & 0.282 (9.04e-4) \\
    \bottomrule
    \end{tabular}
    }
\end{minipage}
\begin{minipage}{\textwidth}
    \centering
    \subcaption{Weighted $F_1$ score}
    \resizebox{0.49\linewidth}{!}{
    \def\arraystretch{1.25}%
    \begin{tabular}{P{1.5cm}|P{1.45cm}P{1.45cm}P{1.45cm}P{1.45cm}}
    \toprule
    {\# hidden neurons} & SGD Training & \SVI{} Training & SGD Test & \SVI{} Test \\
    \midrule
    32 & 0.704 (6.62e-3) & 0.727 (4.36e-3) & 0.659 (1.46e-2) & 0.704 (1.15e-2) \\
    64 & 0.719 (4.57e-3) & 0.737 (1.59e-3) & 0.689 (1.16e-2) & 0.728 (1.14e-3) \\
    128 & 0.706 (3.50e-3) & 0.734 (2.41e-3) & 0.664 (7.73e-3) & 0.718 (9.11e-4) \\
    \bottomrule
    \end{tabular}
    }
\end{minipage}
\end{table}
% Solar convergence
\begin{figure}[!t]
    \centering
     \begin{minipage}{\textwidth}
     \includegraphics[width=\textwidth]{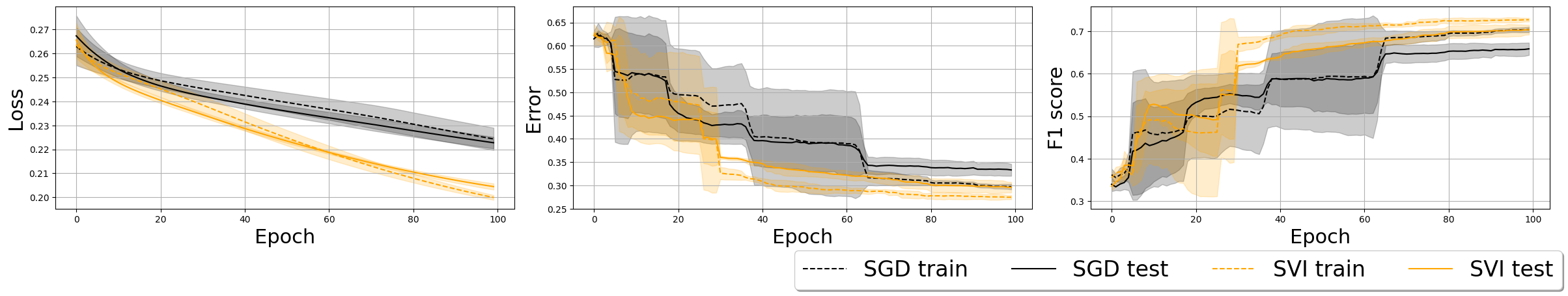}
     \subcaption{32 hidden neurons. Left to right: MSE loss, classification error, and weighted $F_1$ score.}
    \end{minipage}
     \begin{minipage}{\textwidth}
    \centering
     \includegraphics[width=\textwidth]{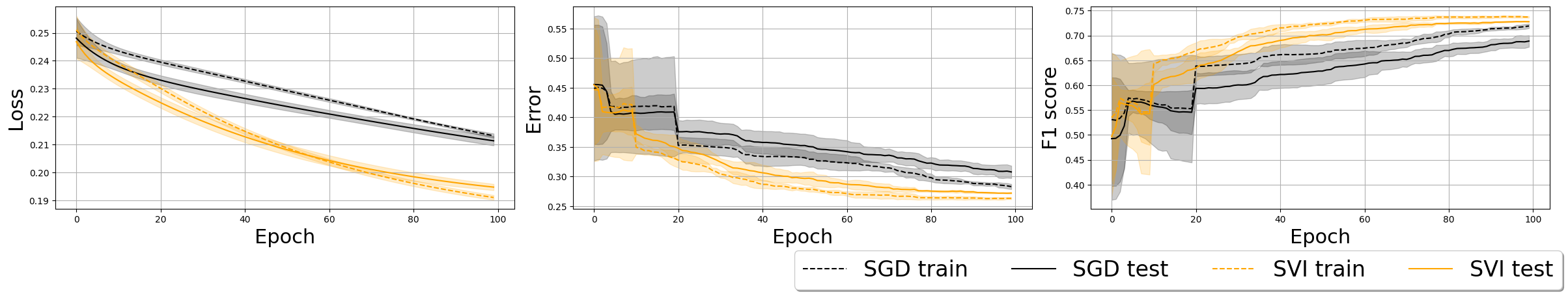}
     \subcaption{64 hidden neurons. Left to right: MSE loss, classification error, and weighted $F_1$ score.}
    \end{minipage}
     \begin{minipage}{\textwidth}
    \centering
     \includegraphics[width=\textwidth]{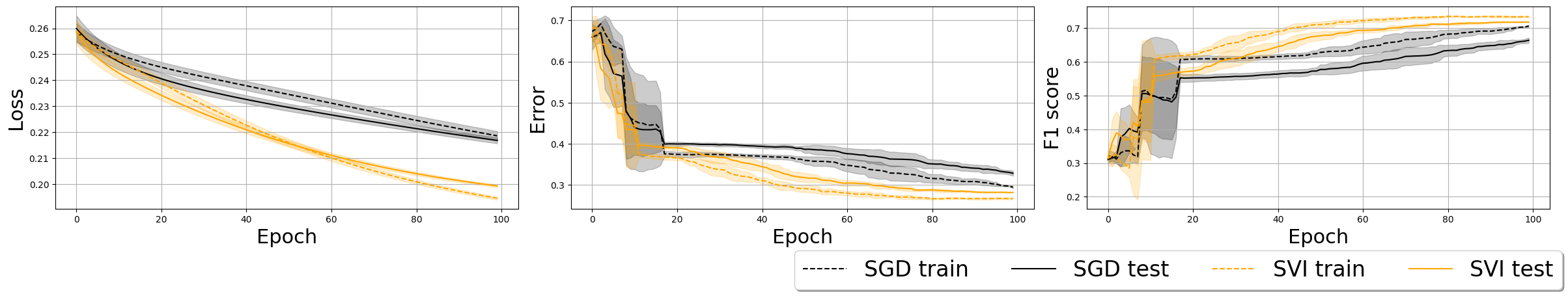}
     \subcaption{128 hidden neurons. Left to right: MSE loss, classification error, and weighted $F_1$ score.}
    \end{minipage}
     \caption{Solar ramping event detection under various hidden neurons. Results are plotted with one standard error bars over three independent trials.}
     \label{fig:append_solar}
\end{figure}

% Real-data Traffic
\begin{table}[!t]
\centering
\cprotect \caption{Traffic data multi-class anomaly detection under varying sizes of the GNN. Entries in brackets indicate standard deviation over three independent initializations of model parameters.}
\label{tab:traffic_tab}
\begin{minipage}{0.49\textwidth}
    \centering
    \subcaption{MSE loss}
    \resizebox{\linewidth}{!}{
    \def\arraystretch{1.25}%
    
    \begin{tabular}{P{1.5cm}|P{1.45cm}P{1.45cm}P{1.45cm}P{1.45cm}}
\toprule
{\# hidden neurons} & SGD Training & \SVI{} Training & SGD Test & \SVI{} Test \\
\midrule
32 & 0.529 (2.84e-2) & 0.475 (7.64e-3) & 0.529 (2.51e-2) & 0.477 (5.63e-3) \\
64 & 0.471 (8.25e-3) & 0.457 (6.51e-3) & 0.473 (7.67e-3) & 0.458 (5.19e-3) \\
128 & 0.447 (2.70e-3) & 0.445 (1.84e-3) & 0.448 (2.44e-3) & 0.445 (1.98e-3) \\
\bottomrule
\end{tabular}

    }
\end{minipage}
\begin{minipage}{0.49\textwidth}
    \centering
    \subcaption{Classification error}
    \resizebox{\linewidth}{!}{
    \def\arraystretch{1.25}%
\begin{tabular}{P{1.5cm}|P{1.45cm}P{1.45cm}P{1.45cm}P{1.45cm}}
\toprule
{\# hidden neurons} & SGD Training & \SVI{} Training & SGD Test & \SVI{} Test \\
\midrule
32 & 0.401 (2.49e-2) & 0.367 (1.24e-2) & 0.404 (2.22e-2) & 0.371 (9.59e-3) \\
64 & 0.344 (9.52e-3) & 0.339 (5.86e-3) & 0.349 (1.12e-2) & 0.345 (6.66e-3) \\
128 & 0.334 (1.64e-3) & 0.334 (2.00e-3) & 0.335 (2.03e-3) & 0.334 (3.27e-3) \\
\bottomrule
\end{tabular}

    }
\end{minipage}
\begin{minipage}{\textwidth}
    \centering
    \subcaption{Weighted $F_1$ score}
    \resizebox{0.49\linewidth}{!}{
    \def\arraystretch{1.25}%
\begin{tabular}{P{1.5cm}|P{1.45cm}P{1.45cm}P{1.45cm}P{1.45cm}}
\toprule
{\# hidden neurons} & SGD Training & \SVI{} Training & SGD Test & \SVI{} Test \\
\midrule
32 & 0.594 (2.89e-2) & 0.629 (1.67e-2) & 0.589 (2.76e-2) & 0.626 (1.29e-2) \\
64 & 0.655 (9.91e-3) & 0.660 (6.05e-3) & 0.651 (1.14e-2) & 0.655 (6.71e-3) \\
128 & 0.665 (1.77e-3) & 0.665 (1.97e-3) & 0.664 (2.16e-3) & 0.666 (3.36e-3) \\
\bottomrule
\end{tabular}

    }
\end{minipage}
\end{table}
% Traffic convergence
\begin{figure}[!t]
    \centering
     \begin{minipage}{\textwidth}
     \includegraphics[width=\textwidth]{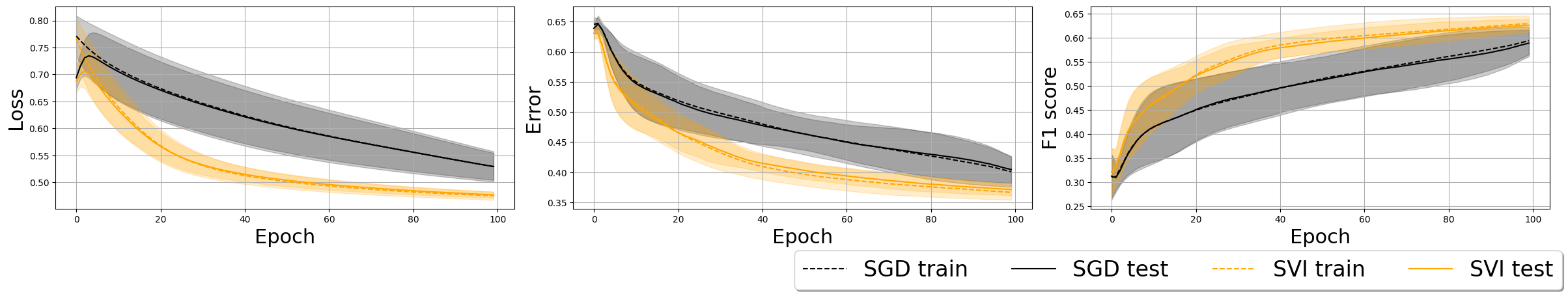}
     \subcaption{32 hidden neurons. Left to right: MSE loss, classification error, and weighted $F_1$ score.}
    \end{minipage}
     \begin{minipage}{\textwidth}
    \centering
     \includegraphics[width=\textwidth]{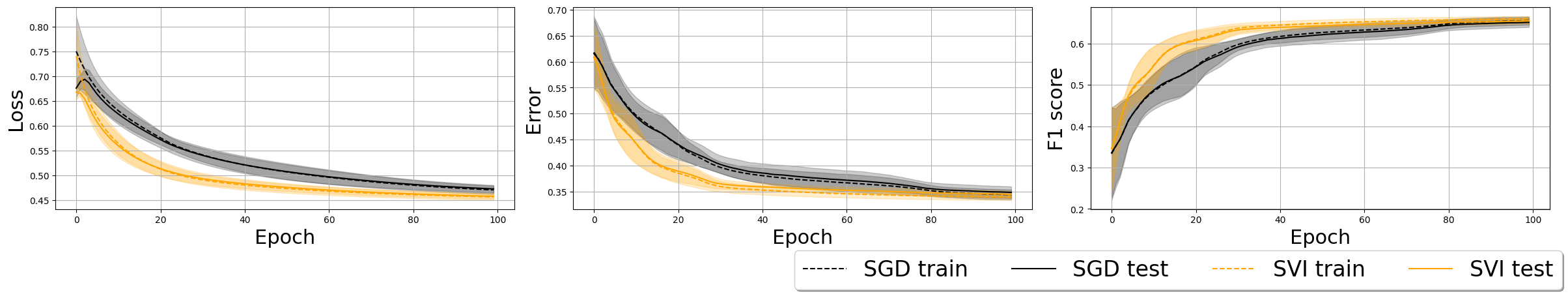}
     \subcaption{64 hidden neurons. Left to right: MSE loss, classification error, and weighted $F_1$ score.}
    \end{minipage}
     \begin{minipage}{\textwidth}
    \centering
     \includegraphics[width=\textwidth]{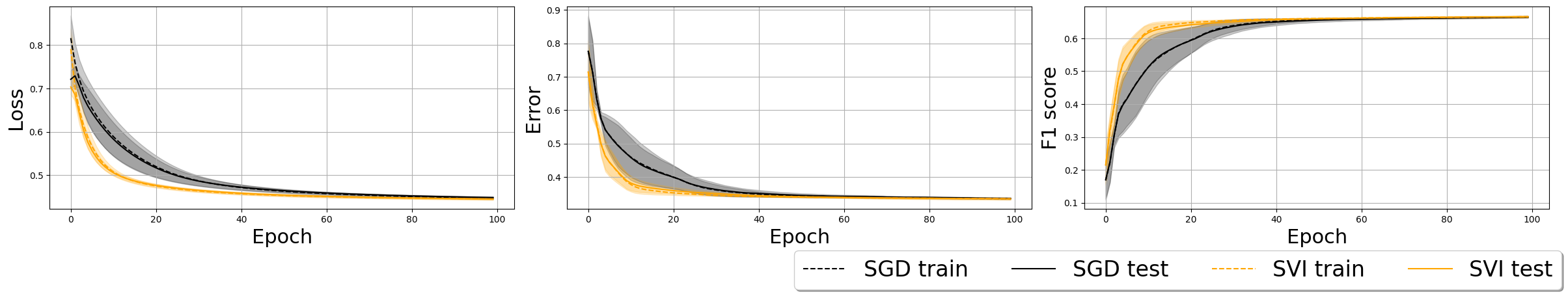}
     \subcaption{128 hidden neurons. Left to right: MSE loss, classification error, and weighted $F_1$ score.}
    \end{minipage}
     \cprotect \caption{Traffic data multi-class anomaly detection under various hidden neurons. Results are plotted with one standard error bar over three independent trials.}
     \label{fig:append_traffic}
\end{figure}

\subsection{Real-data network prediction}\label{sec:real_data}

For real-data experiments, we use larger networks (i.e., three layers or more). Examples include spatial-temporal anomaly detection on network data, graph nodal classification on large graphs, and image classification. The output layer uses sigmoid (binary classification) or softmax (multi-class classification).

\subsubsection{Binary solar ramping prediction}\label{sec:solar} 

The raw solar radiation data are retrieved from the National Solar Radiation Database for 2017 and 2018. We consider data from 10 cities downtown in California and from 9 locations in downtown Los Angeles, where each city or location is a node in the network. In this experiment, the goal is to identify ramping events within a network of solar sensors, where ramping events are defined over abrupt changes in the sensor inputs. Thus, $Y^{(i)}=1$ if node $i$ at day $t$ experiences a ramping event. We define feature $X^{(i)}$ as the collection of past $d$ days of observation and pick $d=5$. We estimate edges via a $k$-nearest neighbor approach based on the correlation between training ramping labels, with $k=4$. Data in 2017 are used for training ($N=360$) and the rest for testing ($N_1=365$), and we let $B=30$, $E=100$, and ${\textsf{lr}}=0.001$.

Table \ref{tab:solar_tab} shows that under varying numbers of hidden nodes in the GNN, \SVI{} consistently reaches lower test classification error and higher test weighted $F_1$ scores; the $F_1$ scores are weighted by support (the number of true instances for each label). Figure \ref{fig:append_solar} shows faster intermediate convergence results by \SVI{} in terms of both metrics.

% OGB table 
\begin{table}[!b]
\cprotect \caption{Classification accuracies on the large \texttt{ogbn-arxiv} dataset under varying sizes of the GNN. ``Initial'' (resp. ``Final'') results indicate prediction accuracies after training 100 (resp. 1000) epochs. Entries in brackets show standard deviation over three independent trials.}
\label{ogb_table_main}
\begin{center}
\vspace{-0.25in}
\resizebox{0.7\textwidth}{!}{
\def\arraystretch{1.25}
    \begin{tabular}{P{1.5cm}|P{1.5cm}P{1cm}P{1cm}P{1cm}|P{1cm}P{1cm}P{1cm}}
\toprule
 &  & \multicolumn{3}{c}{SVI} & \multicolumn{3}{c}{SGD} \\
 \# hidden neurons & result type &         Train &         Valid &          Test &         Train &         Valid &          Test  \\
\midrule
\multirow{2}{*}{128} & Initial & 39.64 (1.99) &  39.52 (1.84) &  39.83 (1.95) &  6.95 (4.36) &  7.05 (4.37) &  6.91 (4.31) \\

& Final & 63.55 (0.25) &  63.44 (0.26) &  63.47 (0.23) &  51.62 (2.22) &  51.38 (2.15) &  51.63 (2.29) \\
\hline
\multirow{2}{*}{256} & Initial & 52.02 (0.95) &  51.84 (1.04) &  52.02 (1.01) &  23.38 (3.86) &  23.35 (3.9) &  23.43 (3.86) \\

& Final & 66.56 (0.08) &  66.2 (0.13) &  66.26 (0.09) &  59.24 (1.56) &  59.14 (1.59) &  59.1 (1.48) \\

\hline
\multirow{2}{*}{512} & Initial & 57.88 (0.36) &  57.57 (0.39) &  57.68 (0.42) &  33.55 (2.42) &  33.46 (2.58) &  33.64 (2.46) \\

& Final & 69.12 (0.13) &  68.52 (0.07) &  68.72 (0.07) &  64.28 (0.77) &  63.99 (0.71) &  64.07 (0.88) \\

\bottomrule
\end{tabular}}
\end{center}    
\end{table}
% OGB figure 
\begin{figure}[h!]
    \centering
    \begin{minipage}{0.49\textwidth}
    \centering
     \includegraphics[width=\linewidth]{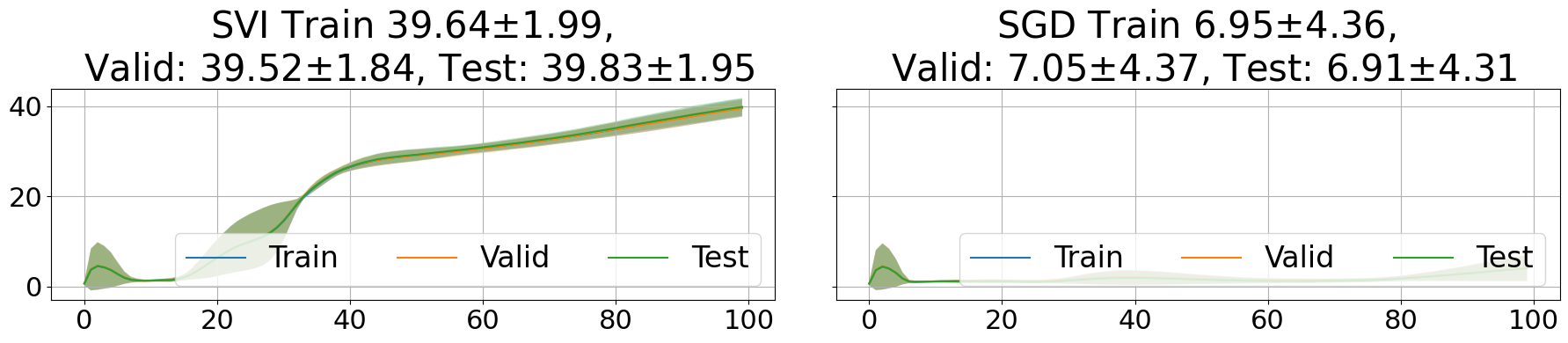}
     \subcaption{128 hidden neuron, initial}
    \end{minipage}
    \begin{minipage}{0.49\textwidth}
    \centering
     \includegraphics[width=\linewidth]{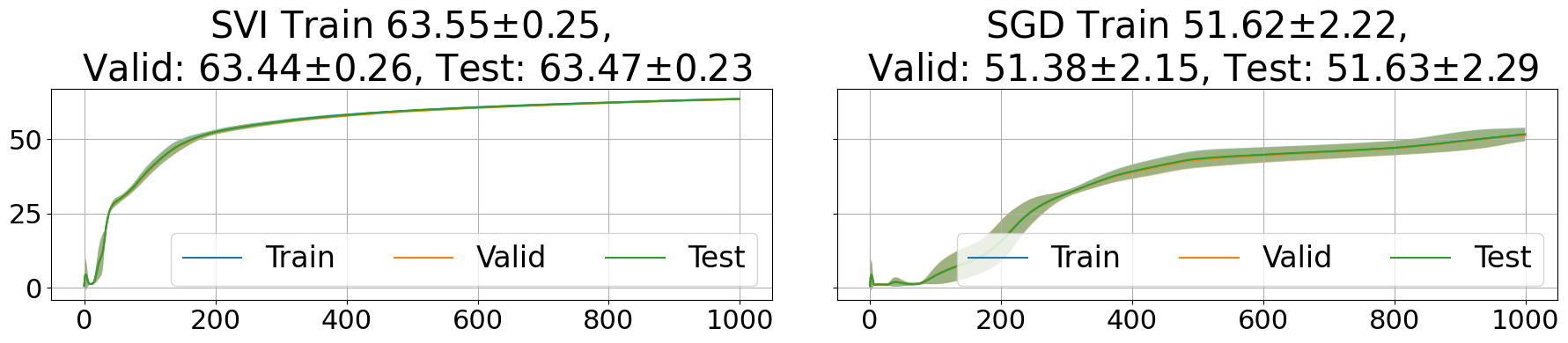}
     \subcaption{128 hidden neuron, full}
    \end{minipage}
    \begin{minipage}{0.49\textwidth}
    \centering
     \includegraphics[width=\linewidth]{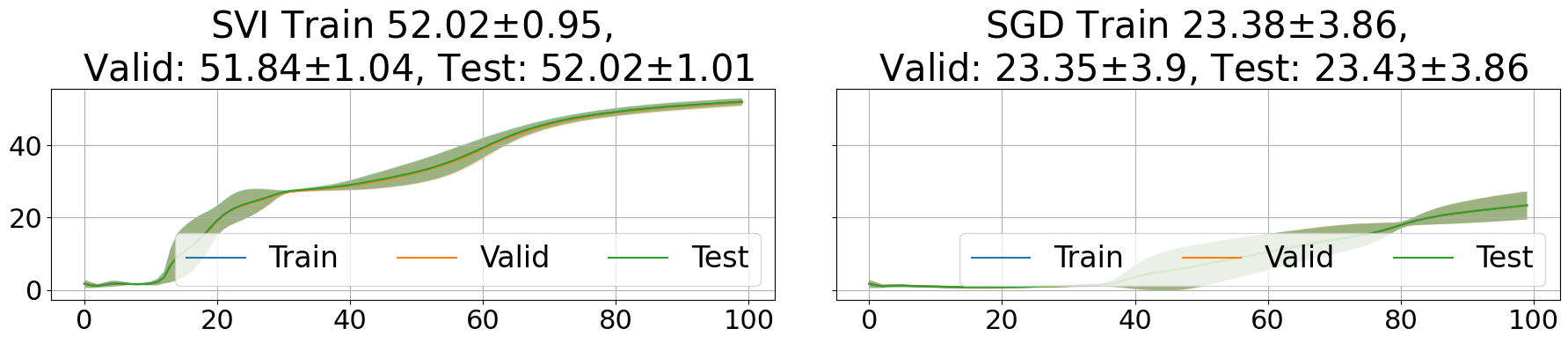}
     \subcaption{256 hidden neuron, initial}
    \end{minipage}
    \begin{minipage}{0.49\textwidth}
    \centering
     \includegraphics[width=\linewidth]{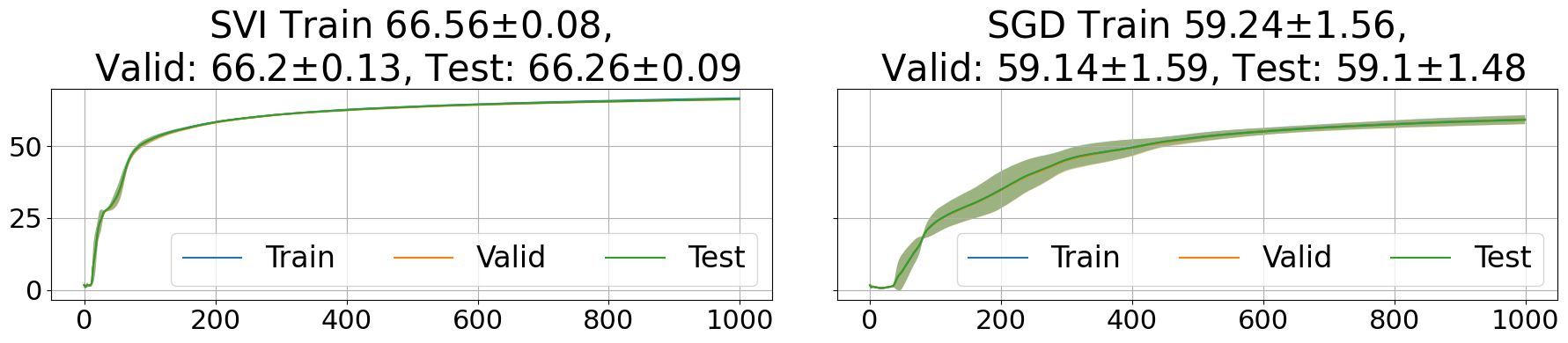}
     \subcaption{256 hidden neuron, full}
    \end{minipage}
    \begin{minipage}{0.49\textwidth}
    \centering
     \includegraphics[width=\linewidth]{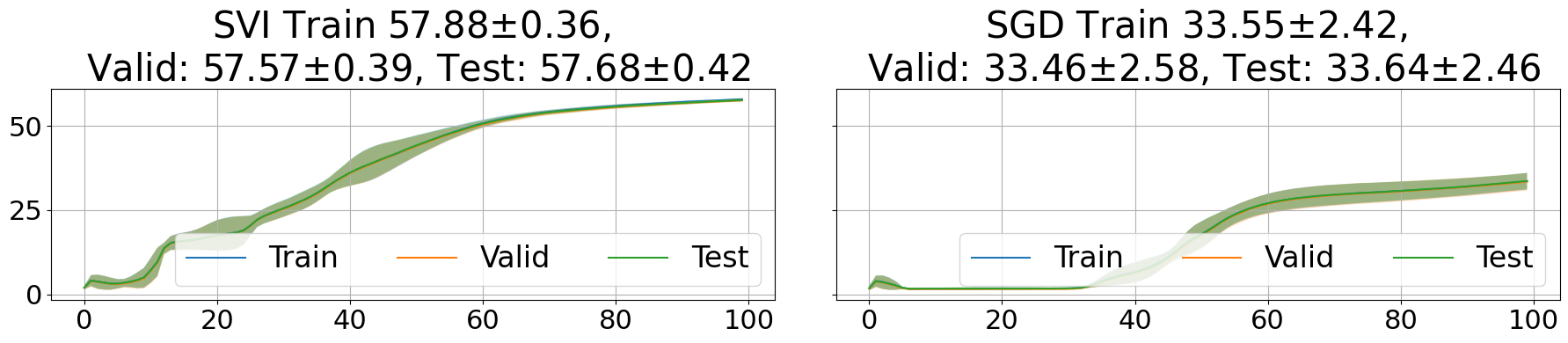}
     \subcaption{512 hidden neuron, initial}
    \end{minipage}
    \begin{minipage}{0.49\textwidth}
    \centering
     \includegraphics[width=\linewidth]{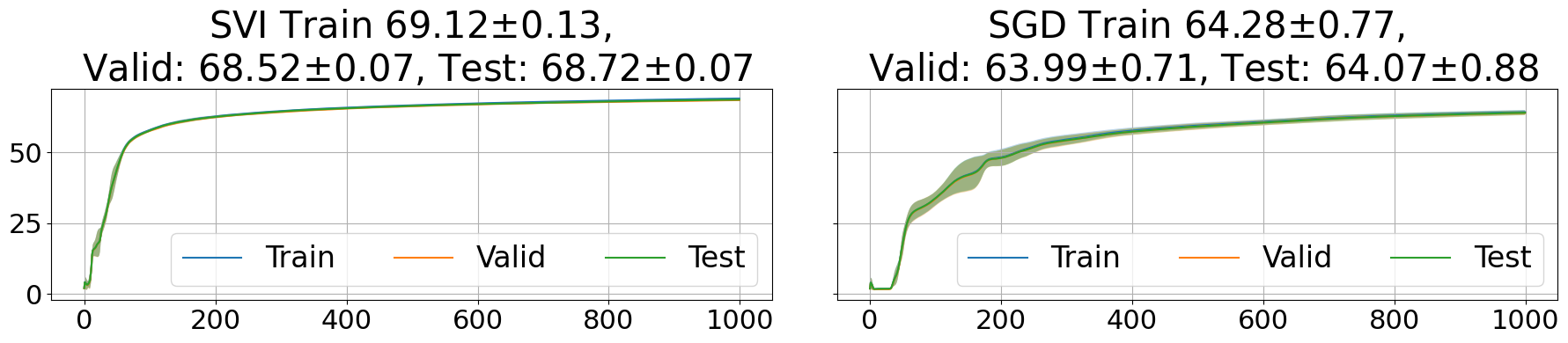}
     \subcaption{512 hidden neuron, full}
    \end{minipage}
     \cprotect \caption{Classification accuracies on the large-scale \texttt{ogbn-arxiv} dataset. We visualize the prediction accuracies during the initial (i.e., 100 epochs) and full (i.e., 1000 epochs) training stages.}
     \label{ogb_fig}
\end{figure}

\vspace{0.1in}

\subsubsection{Multi-class traffic flow anomaly detection}\label{sec:traffic} 

The raw bi-hourly traffic flow data are from the California Department of Transportation, where we collected data from 20 non-uniformly spaced traffic sensors in 2020. This experiment aims to identify multi-class bi-hourly anomalous traffic flow observations in a sensor network. 
Thus, we let $Y^{(i)}=1$ (resp. 2) if the current traffic flow lies outside the upper (resp. lower) 90\% quantile over the past four days of traffic flow of its nearest four neighbors based on sensor proximity. As before, we define feature $X^{(i)}$ as the collection of past $d$ days of observation and set $d=4$, where the edges include the nearest five neighbors based on location. Data in the first nine months are training data (e.g., $N=6138$) and the rest for testing ($N_1=2617$). We let $B=600$, $E=100$, and ${\textsf{lr}}=0.001$.

Table \ref{tab:traffic_tab} shows that \SVI{} consistently reaches lower test classification error and higher test weighted $F_1$ scores. This is consistent with the previous observations and holds with varying sizes of the GNN. Figure \ref{fig:append_traffic} shows faster intermediate convergence results by \SVI{} in terms of both metrics.

\subsubsection{Multi-class large-scale OGB node classification}\label{sec:ogb} 

We further demonstrate the applicability of \SVI \ on the large \texttt{ogbn-arxiv} graph provided by the Open Graph Benchmark (OGB) \citep{hu2020ogb,hu2021ogblsc}. The graph is much larger than earlier examples: graph nodes are papers to be classified, and edges denote citation among papers; it has approximately 170 thousand nodes, 1.16 million edges, 128-dimensional node features, and 40 node classes. These numbers are significantly larger than those in earlier examples. We train four-layer GNN models with varying numbers of hidden nodes, where these GNN models are wider and deeper than the earlier ones we trained. We train for fixed $E=1000$ epochs where $B$ equals all training nodes. The learning rate ${\textsf{lr}}=0.001$.

Table \ref{ogb_table_main} compares \SVI{} against SGD under various network sizes. We see that \SVI{} consistently reaches higher initial and final classification accuracies on training, validation, and test data. Figure \ref{ogb_fig} shows the convergence of prediction accuracy along training epochs, where \SVI{} consistently remains higher in terms of prediction accuracies along training epochs. We believe such faster convergence can benefit large-scale experiments, where it is computationally demanding to run many training epochs.

% Vision plot
\begin{figure}[!t]
     \begin{minipage}{\textwidth}
         \includegraphics[width=\linewidth]{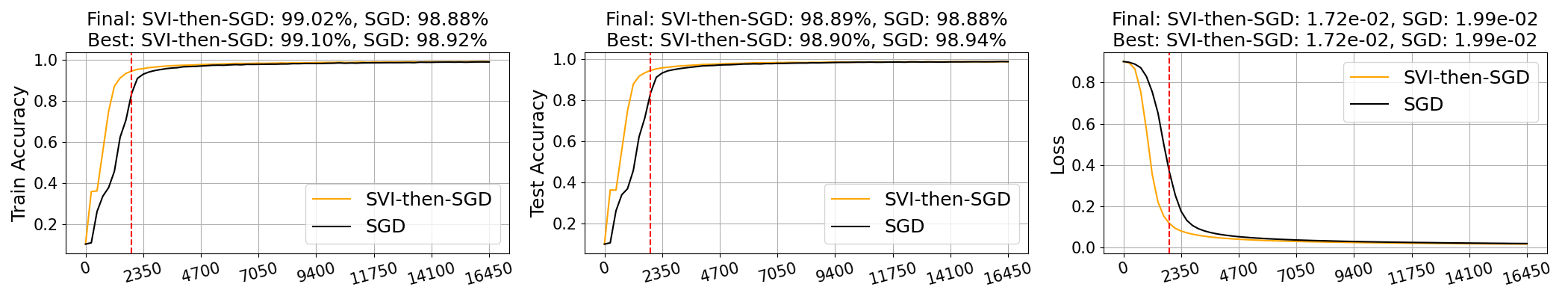}
         \subcaption{MNIST by LeNet: training accuracy (left), test accuracy (middle), and training loss (right).}
     \end{minipage}
     \begin{minipage}{\textwidth}
         \includegraphics[width=\linewidth]{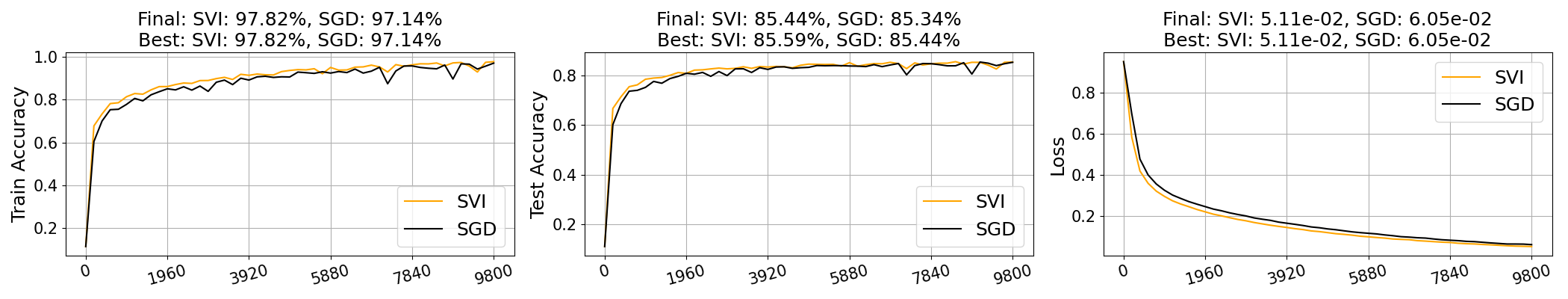}
         \subcaption{CIFAR10 by VGG-16: training accuracy (left), test accuracy (middle), and training loss (right).}
     \end{minipage}
     \cprotect \caption{Classification accuracies and training losses by both methods. We plot the metrics over training batches in each sub-figure. In the title, ``Final'' represents the metric at the end of all training epochs, and ``Best'' represents the highest/lowest metric throughout training epochs. On MNIST, training before the dashed-dotted red line is by \SVI{}, and we continue to train the SVI-warm-started model afterward by SGD to reach the optimal performance.}
     \label{fig:vision}
\end{figure}

\subsubsection{Image classification datasets}\label{sec:img} 

We compare \SVI \ with the gradient-based method on training image classifiers on the MNIST dataset \citep{MNIST} and the CIFAR-10 dataset \citep{CIFAR10}. On MNIST, we train the LeNet \citep{LENET} model, which consists of two convolutional layers as image feature extractors, followed by prediction with three fully-connected layers. On CIFAR10, we train the VGG-16 model \citep{vgg16}, which has 16 convolutional layers. The \textsf{lr} is fixed as 0.005 throughout training, and we fix the batch size $B$ to be 128.

Figure \ref{fig:vision} shows training and test metrics over training batches. Note that on MNIST, we train the initial 10\% of total training batches by \SVI{} and the rest 90\% by SGD, and we call this hybrid technique \SVI{}-then-SGD. We do so because \SVI{} shows fast initial convergence, yet to reach the optimal performance after the initial stages, gradient-based methods can be more desirable. On MNIST, the hybrid approach shows faster initial convergence than SGD due to the use of \SVI{} and reaches higher accuracy and lower loss by the end of all training batches. On CIFAR-10, the performance of \SVI{} is also competitive.

\section{Conclusion}\label{sec:conclusion}
%\vspace{-.1in}

We have investigated how a new monotone VI approach can be useful in training neural networks, with prediction guarantees for last-layer fine-tuning and competitive performance as SGD. We specifically provide theoretical justifications when training parameters in the last layer, assuming previous layers are frozen. On various synthetic and real-data examples, our \SVI{} improves efficiency in the early stage of training and reaches competitive/better final performance as SGD.

The work can be extended in several directions. Testing the performance of \SVI{} on training large models with high-dimensional data is important. Theoretically, we wish to extend the analyses to multi-layer network training. Application-wise, addressing a broader range of problems in GNN is helpful, including edge and graph classification \citep{Zhou2020GraphNN} and various regression problems.

\section*{Acknowledgements}
This work is partially supported by an NSF CAREER CCF-1650913, NSF DMS-2134037, CMMI-2015787, CMMI-2112533, DMS-1938106, DMS-1830210, and the Coca-Cola Foundation. XC is also partially supported by NSF DMS-2237842 and Simons Foundation.

\bibliography{references}
\bibliographystyle{plainnat}

\appendix

\section{Proofs}\label{sec:theory_append}

\begin{proof}[Proof of Lemma \ref{lem1}]
We first verify the monotonicity of $F$ by considering vectorized parameters, where for a matrix $A\in \R^{m\times n}$, $\text{vec}(A) \in \R^{mn}$ by stacking vertically columns of $A$. Note that this vectorization is simply used to make sure the Euclidean inner product between $F$ and $\theta_L$ is well-defined; the fundamental meaning of $\theta_L$ (e.g., as the channel-mixing coefficient) remains unchanged. With an abuse of notation, we use the same $\theta_L$ and $\bTheta$ to denote the vectorized parameter and the corresponding parameter space. For any $\theta_{1,L}, \theta_{2,L} \in \bTheta$,
\begin{align*}
   & \langle F(\theta_{1,L})-F(\theta_{1,L}), \thetalower{1,L}-\thetalower{2,L}
     \rangle \\
    & = \langle \EE_X\{\eta^*(X)^{\intercal}(\phi_L(\eta^*(X)\thetalower{1,L})-\phi_L(\eta^*(X)\thetalower{2,L})\}, \thetalower{1,L}-\thetalower{2,L} \rangle\\
    & =\EE_X\{ (\phi_L(\eta^*(X)\thetalower{1,L})-\phi_L(\eta^*(X)\thetalower{2,L}))^T (\eta^*(X)\thetalower{1,L}-\eta^*(X)\thetalower{2,L})\} \\
    & \geq \lambda_{\min}(\nabla \phi_L) \EE_X\{ \|\eta^*(X)\thetalower{1,L}-\eta^*(X)\thetalower{2,L}\|^2_2\} \\ 
    & \geq \underbrace{ \lambda_{\min}(\nabla \phi_L) \lambda_{\min}(\EE_X\{\eta^*(X)^{\intercal}\eta^*(X)\})}_{\kappa} \|\thetalower{1,L}-\thetalower{2,L}\|_2^2\},
\end{align*}
where $\kappa$ is defined in \eqref{eq:modulus}. The first equality uses the fact that the $Y$ part is cancelled and the first inequality holds when $\phi_L$ is \textit{continuously differentiable} on its domain. 

We then verify the $K_2$-Lipschitz continuity of $F$. For any $\theta_{1,L}, \theta_{2,L} \in \bTheta$,
\begin{align*}
    \|F(\theta_{1,L})-F(\theta_{2,L})\|_2
    &=\EE_X\{\| \eta^*(X)^{\intercal} (\phi_L(\eta^*(X)\thetalower{1,L})-\phi_L(\eta^*(X)\thetalower{2,L})\|_2\}\\
    & \leq \EE_X\{ \|\eta^*(X)^{\intercal}\|_2 \|\phi_L(\eta^*(X)\thetalower{1,L})-\phi_L(\eta^*(X)\thetalower{2,L})\|_2\} \\
    & \leq  K\EE_X\{ \|\eta^*(X)\|_2 \|\eta^*(X)\thetalower{1,L}-\eta^*(X)\thetalower{2,L}\|_2\}\\
    & \leq \underbrace{K \EE_X\{ \|\eta^*(X)\|^2_2\}}_{K_2} \|\thetalower{1,L}-\thetalower{2,L}\|_2^2,
\end{align*}
where $K_2$ has been defined. We repeated used the Cauchy–Schwarz inequality and the last inequality relies on the assumption that $\phi_L$ is $K$-Lipschitz continuous.

Note that given random samples $\{X_1,\ldots, X_N\}$, the quantities $\lambda_{\min}(\EE_X\{\eta^*(X)^{\intercal}\eta^*(X)\})\}$ and $\EE_X\{ \|\eta^*(X)\|^2_2\}$ can be empirically approximated by sample averages. Thus, under $\theta^*_L$,
\begin{align*}
    F(\theta^*_L)
    &=\EEXY{\eta^{\intercal}[\phi_L (\eta^*(X)\theta^*_L)-Y].}\\
    &=\EEXY{\eta^*(X)^{\intercal}[\phi_L (\eta^*(X)\theta^*_L)-\phi_L (\eta^*(X)\theta^*_L)}=0,
\end{align*}
where we used the fact tht $\condexp{Y}{X}=\phi_L (\eta^*(X)\theta^*_L).$

\end{proof}
\begin{proof}[Proof of Lemma \ref{thm:alg_guarantee}]
The proof employs classical techniques when analyzing the convergence of projection descent in stochastic optimization, which appear in \citep[Proposition 3.2]{VI_est}.

First, for any $\theta_L \in \bTheta$, 
\begin{align*}
    \mathbb{E}_{(X,\YTheta)} \{\|\eta^*(X)^{\intercal}\phi_L(\eta^*(X)\theta_L)\|_2\}
    & = \mathbb{E}_{X} \{\|\mathbb{E}_{\YTheta}\{\eta^*(X)\YTheta\}\|_2\}\\
    & \leq \mathbb{E}_{X} \mathbb{E}_{\YTheta}\{\|\eta^*(X)\YTheta\}\|_2\} && [\text{Jensen's Inequality}] \\
    & = \mathbb{E}_{(X,\YTheta)} \{\|\eta^*(X)\YTheta\}\|_2\} \leq M.
\end{align*}
By the form of $F$, we then have that $\EEXY{\|F(\theta_L)\|_2^2} \leq 4M^2$ for any $\theta_L$.

Next, note that each $\hatThetat{t}$ is a deterministic function of $\mathcal D^{\rm Tr}$. Define the difference of estimation and its expected value as  
\[
D_t(\mathcal D^{\rm Tr})=\frac{1}{2} \normsquare{\hatThetat{t}-\thetastar{L}}, \ d_t= \mathbb{E}_{\mathcal D^{\rm Tr} }\{D_t(\mathcal D^{\rm Tr})\}.
\]
As a result,
\begin{align*}
    D_t(\mathcal D^{\rm Tr}) 
    & = \frac{1}{2} \normsquare{\operatorname{Proj}_{\bTheta}\left[\hatThetat{t-1}-\gamma_{t} \Femp^T(\hatThetat{t-1})-\thetastar{L}\right]} \\
    & \leq \frac{1}{2} \normsquare{\hatThetat{t-1}-\gamma_{t} \Femp^T(\hatThetat{t-1})-\thetastar{L}} \quad \commentalign{The projection is a contraction} \\
    & = \frac{1}{2} \normsquare{\hatThetat{t-1}-\thetastar{L}}-\gamma_t \Femp^T(\hatThetat{t-1})(\hatThetat{t-1}-\thetastar{L})+\frac{1}{2}\gamma_t^2 \normsquare{\Femp^T(\hatThetat{t-1})}.
\end{align*}
Taking expectation of both sides with respect to $\mathcal D^{\rm Tr} $ yields 
\begin{align*}
    d_t 
    & \leq \frac{1}{2} d_{t-1}-\gamma_t \mathbb{E}_{\mathcal D^{\rm Tr} } \left[\Femp^T(\hatThetat{t-1})(\hatThetat{t-1}-\thetastar{L}) \right]+2\gamma_t^2M^2 \\
    & \leq (1-2\kappa\gamma_t)d_{t-1}+2\gamma_t^2M^2,
\end{align*}
where the last inequality follows by noting that $\Femp$ is an unbiased estimator of $F$, which satisfies 
\[
F(\theta_L)^T(\theta_L-\thetastar{L})\geq \kappa \|\theta_L-\thetastar{L}\|_2,
\]
due to $F(\thetastar{L})=0$. Then, using triangle inequality yields the result.

Lastly, we prove by induction that if we define $R=(2M^2)/\kappa^2, \ \gamma_t=1/\kappa(t+1)$, we have
\[
    d_t \leq \frac{R}{t+1}.
\]
\noindent \textit{(Base case when $t=0$.)} Let $B$ be the $\|\cdot\|_2$ diameter of $\bTheta$ (e.g., $\normsquare{\theta_1-\theta_2}\leq B^2 \ \forall(\theta_1,\theta_2) \in \bTheta.$ Denote $\theta_L^+, \theta_L^- \in \B$ to satisfy $\normsquare{\theta_L^+-\theta_L^-}=B^2.$
By the definition of $\kappa$,
\[
\langle F(\theta_L^+)-F(\theta_L^-),\theta_L^+-\theta_L^- \rangle \geq \kappa \|\theta_L^+-\theta_L^-\|_2^2=\kappa B^2.
\]
Meanwhile, the Cauchy-Schwarz inequality yields
\[
\langle F(\theta_L^+)-F(\theta_L^-),\theta_L^+-\theta_L^- \rangle= \langle \eta^*(X)(\phi_L(\eta^*(X))\theta_L^+)-\eta^*(X)(\phi_L(\eta^*(X))\theta_L^-),\theta_L^+-\theta_L^- \rangle \leq 2MB.
\]
Thus, $B\leq 2M/\kappa$. As a result, $B^2/2\leq 2M^2/\kappa^2=R$. Because $d_0=\normsquare{\hatThetat{0}-\thetastar{L}}\leq B^2,$ 
\[
d_0 \leq 2R = \frac{4M^2}{\kappa^2}.
\]
\noindent \textit{(The inductive step from $t-1$ to $t$.)} Note that by the definition of $\gamma_t$, $\kappa\gamma_t=1/(t+1)\leq 1/2$. Thus,
\begin{align*}
    d_t & \leq (1-2\kappa\gamma_t)d_{t-1}+2\gamma_t^2M^2 \\
    & = \frac{R}{t} (1-\frac{2}{t+1})+\frac{R}{(t+1)^2} \leq \frac{R}{t+1},
\end{align*}
whereby the proof is complete by the definition of $d_t$ and $R$.
\end{proof}

\begin{proof}[Proof of Theorem \ref{thm:generalization_err}]
We have that
\begin{align*}
    & \mathbb E_{X,\mathcal D^{\rm Tr}}\{\|\widehat{Y}(X, \widehat{\theta}_L^{(T)})-\condexp{Y}{X}\|^2_2\} \\
    = & \mathbb E_{X,\mathcal D^{\rm Tr}} \{ 
    \|\phi_L(X_{L}^*\widehat{\theta}_L^{(T)})- \phi_L(X_{L}^* \theta^*_L)\|^2_2 \} \\
    \overset{(i)}{\leq} & \mathbb E_{X,\mathcal D^{\rm Tr}} \{
    K^2 \|X_{L}^*\widehat{\theta}_L^{(T)}- X_{L}^*\theta^*_L\|^2_2
    \} \\
    = & \mathbb E_{X,\mathcal D^{\rm Tr}} \{ 
    K^2 [\widehat{\theta}_L^{(T)}-\theta^*_L]^TX_{L}^{* \intercal}X_{L}^*[\widehat{\theta}_L^{(T)}-\theta^*_L]
    \} \\
    \overset{(ii)}{=} & \mathbb E_{\mathcal D^{\rm Tr}} \{ 
    K^2 [\widehat{\theta}_L^{(T)}-\theta^*_L]^T
    \mathbb E_X \{X_{L}^{* \intercal}X_{L}^*\}
    [\widehat{\theta}_L^{(T)}-\theta^*_L]
    \} \\
    \leq & K^2 \mathbb  \lambda_{\max}(E_X\{X_{L}^{* \intercal} X_{L}^*\}) \mathbb E_{\mathcal D^{\rm Tr}} \{ 
    \|\widehat{\theta}_L^{(T)}-\theta^*_L\}^2_2]
    \},
\end{align*}
where (i) holds by $K$-Lipschitz assumption on $\phi_L$ and (ii) holds by the independence of $X$ (as a new feature) and $\widehat{\theta}_L^{(T)}$ (depending on training data only).
We can then use the bound on $\mathbb E_{\mathcal D^{\rm Tr}}\{\|\hat{\theta}^{(T)}_L-\theta_L\|_2\}$ from the previous lemma to complete the proof.
\end{proof}

\begin{proof}[Proof of Theorem \ref{thm:generalization_err_nostrong}] For a given $\theta_L \in \Theta$, let $N_{\bTheta}(\theta_L)=\{y\in \R^{p} | \langle y, \theta'-\theta_L \rangle, \forall \theta' \in \bTheta\}$ denote the normal cone of $\bTheta$ at $\theta_L$. The first step is to bound the expected value of the norm of $F$ evaluated at the stochastic OE estimate. This bound results from \citep[Theorem 3.8]{kotsalis2022simple}, where for any $\theta_L \in \bTheta$, the residual
\[
\EE[\text{res}(\theta_L)]\leq \epsilon, \ \text{res}(\theta_L)=\min_{y\in -N_{\bTheta}(\theta_L)} \|y-F(\theta_L)\|_2
\]
acts as the termination criteria for the recurrence under a certain choice of the Bregman's distance $V(a,b)$; we let $V(a,b)=\|a-b\|^2_2/2$ in our case. 

Under assumptions on $F$ and choices of step sizes, we can thus restate \citep[Theorem 3.8]{kotsalis2022simple} in our special case as
\[
\EE_{\mathcal D^{\rm Tr}}[\text{res}(\hat{\theta}_L^{(R)})] \leq \frac{3\sigma}{\sqrt{T}}+\frac{12K_2\sqrt{2\|\thetastar{L}\|_2^2+\frac{2\sigma^2}{L^2}}}{\sqrt{T}}.
\]
When $\bTheta$ is the entire space, $N_{\bTheta}(\theta_L)=\{\boldsymbol 0\}$. As a result, $\text{res}(\hat{\theta}_L^{(R)})=\|F(\hat{\theta}_L^{(R)})\|_2$.

Furthermore, recall the fact that for a matrix $A\in \R^{m \times n}$ and vectors $x,x' \in R^n$, we have \[
\|x-x'\|_2\leq \|A(x-x')\|_2/\sigma_{\min}(A),
\]
where $\sigma_{\min}(A)$ denotes the smallest singular value of $A$. Thus, by letting $A=\eta^*(X), x=\widehat{Y}(X, \widehat{\theta}_L^{(R)}), x'=\condexp{Y}{X}$ we have in expectation that 
\[
\EE_{\mathcal D^{\rm Tr}} \{\| \EE_X\{ \sigma_{\min} (\eta^*(X))[\widehat{Y}(X, \widehat{\theta}_L^{(R)})-\condexp{Y}{X}]\}\|_2 \} \leq \EE_{\mathcal D^{\rm Tr}} \|F(\hat{\theta}_L^{(R)})\|_2,
\]
where we used the fact $ F(\hat{\theta}_L^{(R)})=\EEXY{\eta^*(X)^{\intercal}[\phi_L (\eta^*(X)\hat{\theta}_L^{(R)})-Y]}=\EE_X\{\eta^*(X)^{\intercal}[\widehat{Y}(X, \widehat{\theta}_L^{(R)})-\condexp{Y}{X}]]\}.$
\end{proof}

\begin{proof}[Proof of Proposition \ref{prop:equivalence}]
For notation simplicity, denote $\eta=\eta^*(X)$ and $\theta=\theta_L$. Meanwhile, for two vectors $a,b\in \R^n$, the notation $a/b$ denotes the point-wise division. $\boldsymbol 1$ denotes a vector of all 1.

We first consider the binary cross-entropy loss $\Lc(Y,\theta_L)$ defined in \eqref{eq:binaryCE}. Note that we have
\begin{align*}
     \Lc(Y,\theta_L)
     &=-Y\log (\phi_L(\eta^*(X)\theta_L))-(1-Y)\log (1-\phi_L(\eta^*(X)\theta_L))\\
     &= -Y^T \log(\exp(\eta \theta)/(\boldsymbol 1+\exp(\eta \theta))) - (\boldsymbol 1-Y)^T\log(\boldsymbol 1/(\boldsymbol 1+\exp(\eta \theta)))\\
     &= -Y^T\eta \theta+Y^T \log(\boldsymbol 1+\exp(\eta \theta))+(\boldsymbol 1-Y)^T\log(\boldsymbol 1+\exp(\eta \theta))\\
     &= \boldsymbol 1^T \log(\boldsymbol 1+\exp(\eta \theta))-Y^T\eta \theta.
\end{align*}
Thus, the gradient with respect to $\theta$ can be written as 
\[
    \nabla_{\theta} \Lc(Y,\theta_L)
    =\eta^T\frac{\exp(\eta \theta)}{\boldsymbol 1+\exp(\eta \theta)}-\eta^TY = \eta^*(X)^{\intercal}[\phi_L(\eta^*(X)\theta_L)-Y].
\]
Taking expectation thus yields the result.

We next consider the categorical cross-entropy loss $\Lc(Y,\theta_L)$ defined in \eqref{eq:multiCE}. Note that we have 
\begin{align*}
    \Lc(Y,\theta_L)
     &=-e_{Y}^T \log(\phi_L(\eta^*(X)\theta_L))\\
     &=\left [-e_{Y}^T \log\left (\frac{\exp(\eta\theta)}{\boldsymbol 1^T\exp(\eta\theta)}\right)\right ]\\
     &= -e_{Y}^T(\eta\theta)+\log(\boldsymbol 1^T\exp(\eta\theta)).
\end{align*}
Thus, the gradient with respect to $\theta$ can be written as
\begin{align*}
    \nabla_{\theta} \Lc(Y,\theta)
    & = -\eta^T e_{Y}+\eta^T \frac{\exp(\eta\theta)}{\boldsymbol 1^T\exp(\eta\theta)} \\
    & = \eta^T[\phi_L(\eta\theta)-e_{Y}] 
    = \eta^*(X)^{\intercal}[\phi_L(\eta^*(X)\theta_L)-Y],
\end{align*}
where in the definition of $F(\theta)$ in \eqref{eq:operatr}, $Y=e_Y \in \R^{k+1}$ if $Y$ belongs to more than 2 classes. Taking expectation thus yields the result. 
\end{proof}

\end{document}